\documentclass{article} 
\usepackage{iclr2020_conference,times, amsmath,amsfonts,amsthm,amssymb, graphicx,longtable, color}

\usepackage{amsmath,amsfonts,bm}

\def\eqref#1{equation~\ref{#1}}

\def\1{\bm{1}}

\DeclareMathAlphabet{\mathsfit}{\encodingdefault}{\sfdefault}{m}{sl}
\SetMathAlphabet{\mathsfit}{bold}{\encodingdefault}{\sfdefault}{bx}{n}

\DeclareMathOperator*{\argmax}{arg\,max}

\usepackage{hyperref}
\usepackage{url}
\usepackage{subfigure}
\usepackage{color}

\usepackage{color}
\usepackage{mathtools}
\usepackage{bbm}
\usepackage{multirow}
\usepackage{scrextend}
\usepackage{enumitem}

\usepackage{caption}
\usepackage{microtype}

\title{Infinite-horizon Off-Policy Policy Evaluation with Multiple Behavior Policies}

\author{Xinyun Chen$^1$, Lu Wang$^2$, Yizhe Hang$^3$, Heng Ge$^4$ \& Hongyuan Zha$^5$\thanks{Hongyuan zha is on leave from Georgia Institute of technology. Part of this work was done while lu wang, yizhe hang were visiting the Chinese university of hong kong, Shenzhen} \\
\\$^1$Insitute for Data and Decision Analytics, The Chinese University of Hong Kong, Shenzhen
\\$^2$Department of Computer Science, East China Normal University
\\$^3$Department of Computer Science, University of Science and Technology of China
\\$^4$ School of Mathematics and Statistics, Shandong University
\\$^5$Insitute for Data and Decision Analytics, The Chinese University of Hong Kong, Shenzhen \&\\
Georgia Institute of Technology
\\
$^1$\texttt{chenxinyun@cuhk.edu.cn},
$^2$\texttt{luwang@stu.ecnu.edu.cn},\\
$^3$\texttt{hangyhan@mail.ustc.edu.cn},
$^4$\texttt{hengge@mail.sdu.edu.cn},\\
$^5$\texttt{zhahy@cuhk.edu.cn} 
}

 \iclrfinalcopy
\theoremstyle{plain}
\newtheorem{thm}{Theorem}
\newtheorem{proposition}{Proposition}

\newtheorem{corollary}{Corollary}
\newtheorem{assumption}{Assumption}

\begin{document}

\maketitle

\begin{abstract}
We consider off-policy policy evaluation when the trajectory data are generated by multiple behavior policies.
Recent work has shown the key role played by the state or state-action stationary distribution corrections in the
infinite horizon context for off-policy policy evaluation. We propose estimated mixture policy (EMP), a novel class of partially policy-agnostic methods to accurately estimate those quantities. With careful analysis, we show that EMP gives rise to
estimates with reduced variance for estimating the state stationary distribution correction while it also offers a useful induction bias
for estimating the state-action stationary distribution correction. In extensive experiments with both continuous and discrete environments, we demonstrate that our algorithm offers significantly improved accuracy compared to the state-of-the-art methods.
\end{abstract}
\section{Introduction}
In many real-world decision-making scenarios, evaluating a novel policy by directly executing it in the environment is generally costly and can even be downright risky. Examples include evaluating a recommendation policy~\citep{SwaminathanKADL17,zheng2018drn}, a treatment policy~\citep{hirano2003efficient,murphy2001marginal}, and a traffic light control policy~\citep{van2016coordinated}. Off-policy policy evaluation methods (OPPE) utilize a set of previously-collected trajectories (for example, website interaction logs, patient trajectories, or robot trajectories) to estimate the value of a novel decision-making policy without interacting with the environment~\citep{precup2001off,DudikLL11}. For many reinforcement learning applications, the value of the decision is defined in a long- or infinite-horizon, which makes OPPE more challenging.


The state-of-the-art methods for infinite-horizon off-policy policy evaluation rely on learning \textit{(discounted) state stationary distribution corrections} or \textit{ratios}. In particular, for each state in the environments, these methods estimate the likelihood ratio of the long-term probability measure for the state to be visited in a trajectory generated by the \textit{target policy}, normalized by the probability measure generated by the \textit{behavior policy}. This approach can effectively avoid the exponentially high variance compared to the more classic importance sampling (IS) estimation methods~\citep{precup2000eligibility,DudikLL11,hirano2003efficient,wang2017optimal,murphy2001marginal}, especially for infinite-horizon policy evaluation~\citep{liumethod,dualdice,hallak2017consistent}. However, learning state stationary distribution requires detailed information on distributions of the behavior policy, and we call them \textit{policy-aware} methods. As a consequence, policy-aware methods are difficult to apply when off-policy data are pre-generated by multiple behavior policies or when the behavior policy's form is unknown. To address this issue, \citet{dualdice} proposes a \textit{policy-agnostic} method, DualDice, which learns the joint state-action stationary distribution correction that is much higher dimension, and therefore needs more model parameters than the state stationary distribution. Besides, there is no theoretic comparison between policy-aware and policy-agnostic methods. 

 In this paper, we propose a \textit{partially policy-agnostic} method, EMP (estimated mixture policy) for infinite-horizon off-policy policy evaluation with multiple known or unknown behavior policies. EMP is partially policy-agnostic in the since that it does not necessarily require knowledge of the individual behavior policies. Instead, it involves a pre-estimation step to estimate a single mixed policy that will be defined formally later. Like the method in~\citet{liumethod}, EMP also learns the state stationary distribution correction, so it remains computationally cheap and is scalable in terms of the number of behavior policies. Inspired by~\citet{hanna2019importance}, we construct a theoretical bound for the mean square error (MSE) of the stationary distribution corrections learned by EMP. In particular, we show that in the single-behavior policy setting, EMP yields smaller MSE than the policy-aware method. On the other hand, compared to DualDice, EMP learns the state stationary distribution correction of smaller dimension, more importantly the estimation of the mixture policy can be considered as an inductive bias as far as the stationary distribution correction is concerned, and hence could achieve better performance when the pre-estimation is not expensive. In addition, we propose an ad-hoc improvement of EMP, whose theoretical analysis is left for future studies. EMP is compared with both policy-aware and policy-agnostic methods in a set of continuous and discrete control tasks and shows significant improvement.

\section{Background and Related Work}
\subsection{Infinite-horizon Off-policy Policy Evaluation}
We consider a Markov Decision Process (MDP) and our goal is to estimate the infinite-horizon \textit{average reward}. The environment is specified by a tuple $\mathcal{M} =\left\langle S, A, R, T\right\rangle$, consisting of a state space, an action space, a reward function, and a transition probability function. A policy $\pi$ interacts with the environment iteratively, 
starting with an initial state $s_0$. At step $n = 0, 1,...$ , the policy produces a distribution $\pi(\cdot|s_n)$ over the actions $A$, from which an action $a_n$ is sampled and applied to the environment. The environment stochastically produces a scalar reward $r(s_n, a_n)$ and a next state $s_{n+1}\sim T(\cdot|s_n,a_n)$. The infinite-horizon average reward under policy $\pi$ is 
$$R_\pi = \lim_{N\to\infty}\frac{1}{N+1}\sum_{n=0}^N {\cal E} \left[r(s_n,a_n)\right].$$

Without gathering new data, off-policy policy evaluation (OPPE) considers the problem of estimating the expected reward of a target policy $\pi$ via a pre-collected state-action-reward tuples from policies that are different from $\pi$, which are called behavior policies. In our paper, we consider the general setting that the data are generated by multiple behavior policies $\pi_j (j=0,1,..,m)$. Most OPPE literature has focused on the single-behavior-policy case where $m=1$. In this case, we denote the behavior policy by $\pi_0$. Roughly speaking, most OPPE methods can be grouped into two categories: importance-sampling(IS) based OPPE and stationary-distribution-correction based OPPE. 

\subsection{Importance Sampling Policy Evaluation Using Exact and Estimated Behavior Policy}

As for short-horizon off-policy policy evaluation, importance sampling policy evaluation (IS) methods~\citep{precup2001off,DudikLL11,SwaminathanKADL17,PrecupSS00,horvitz1952generalization} have shown promising empirical results. The main idea of importance sampling based OPPE is using importance weighting $\pi/\pi_j$ to correct the mismatch between the target policy $\pi$ and the behavior policy $\pi_j$ that generates the trajectory. 

\cite{li2015toward} and \cite{hanna2019importance} show that using estimated behavior policy in the importance weighting can obtain importance sampling estimation with smaller mean square error (MSE). EMP also uses estimated policy, but there are two key difference between EMP and the previous works: (1) EMP is not an IS-based method, it involves a min-max problem; (2) EMP focuses on multiple-behavior-policy setting while previous works have focused on single-behavior setting.

\subsection{Policy Evaluation via Learning Stationary Distribution Correction}
The state-of-the-art methods for long-horizon off-policy policy evaluation are stationary-distribution-correction based~\citep{liumethod,dualdice,hallak2017consistent}. Let $d_{\pi_0}(s)$ and $d_\pi(s)$ be the stationary distribution of state $s$ under the behavior policy $\pi_0$ and target policy $\pi$ respectively. The main idea is directly applying importance weighting by $\omega = d_\pi/d_{\pi_0}$ on the stationary state-visitation distributions to avoid the exploding variance suffered from IS, and estimate the average reward as
$$R_\pi=\mathbb{E}_{(s,a)\sim d_{\pi}}[r(s,a)]=\mathbb{E}_{(s,a)\sim d_{\pi_0}}\left[\omega(s)\cdot \frac{\pi(a|s)}{\pi_0(a|s)}r(s,a)\right].$$
For example, \citet{liumethod} uses min-max approach to estimate $\omega$ directly from the data. This class of methods require exact knowledge of behavior policy $\pi_0$ and are not straightforward to apply in multiple-behavior-policy setting. Recently, \citet{dualdice} proposes DualDice to overcome such limitation by learning the state-action stationary distribution correction $\omega(s,a) = d_\pi(s)\pi(a|s)/d_{\pi_0}(s)\pi_0(a|s)$.

\section{Single Behavior Policy}\label{sec: single}

We first consider the task of stationary distribution correction learning in the simple case where the data are generated by a single behavior policy as previous state stationary distribution correction methods. To explain the min-max problem formulation of the learning task, we first breifly review the method introduced by \cite{liumethod} in Section \ref{sec: BCD known}, which we shall refer as the BCH method in the rest of the paper. In Section \ref{sec: BCD estimated}, we show that it is beneficial to replace the exact values of the behavior policy in the min-max problem by their estimated values in two folds. First, this extends the method to application setting where the behavior policy is unknown. Second, even when the behavior policy is known with exact values, we prove that the stationary distribution correction learned by the min-max problem with estimated behavior policy has smaller MSE. We will deal with multiple-behavior-policy cases in Section \ref{sec: mis}.

\subsection{Learning Stationary Distribution Correction with Exact Behavior Policy}\label{sec: BCD known}
Assume the data, consisting of state-action-next-state tuples, are generated by a single behavior policy $\pi_0$, i.e. $\mathcal{D}=\{(s_n,a_n,s'_n):n=1,2,...,N\}$. Recall that $d_{\pi_0}$ and $d_\pi$ are the stationary state distribution under the behavior and target policy respectively, and $\omega=d_\pi/d_{\pi_0}$ is the \textit{stationary distribution correction}. In the rest of Section \ref{sec: single}, by slight notation abusion, we also denote $d_{\pi}(s,a)= d_{\pi}(s)\pi(a|s)$, $d_{\pi_0}(s,a)= d_{\pi_0}(s)\pi_0(a|s)$ and $d_{\pi_0}(s,a,s')= d_{\pi_0}(s)\pi_0(a|s)T(s'|a,s)$.  

We briefly review the BCH method proposed by~\cite{liumethod}. As $d_\pi(s)$ is the stationary distribution of $s_n$ as $n \to \infty$ under policy $\pi$, it follows that:
\begin{equation}\label{eq:stationary}
\begin{aligned}
d_\pi(s') &= \sum_{s,a} d_\pi(s)\pi(a|s)P(s'|s,a)=\sum_{s,a}\omega(s)\frac{\pi(a|s)}{\pi_0(a|s)}d_{\pi_0}(s)\pi_0(a|s)T(s'|a,s),\quad \forall s'.
\end{aligned}
\end{equation}
Therefore, for any function $f:S\to\mathbb{R}$,
\begin{equation*}
\sum_{s'}\omega(s')d_{\pi_0}(s')f(s')=\sum_{s,a,s'}\omega(s)\frac{\pi(a|s)}{\pi_0(a|s)}d_{\pi_0}(s)\pi(a|s)T(s'|a,s)f(s').
\end{equation*}
Recall that $d_{\pi_0}(s,a,s')= d_{\pi_0}(s)\pi_0(a|s)T(s'|a,s)$, so $\omega$ and the data sample satisfy the following equation
\begin{equation*}
\mathbb{E}_{(s,a,s')\sim d_{\pi_0}}\left[\left(\omega(s')-\omega(s)\frac{\pi(a|s)}{\pi_0(a|s)}\right)f(s')\right]=0, \text{ for all }f.
\end{equation*}
BCH solves the above equation via the following min-max problem:
\begin{equation}\label{eq: BCD equation}
\min_{\omega}\max_f~ \mathbb{E}_{(s,a,s')\sim{d_{\pi_0}}}\left[\left(\omega(s')-\omega(s)\frac{\pi(a|s)}{\pi_0(a|s)}\right)f(s')\right]^2,
\end{equation}
and use \textit{kernel method} to solve $\omega$. The derivation of kernel method are put in Appendix \ref{appdx: kernel}. 

\subsection{Learning Stationary Distribution Correction with Estimated Behavior Policy}\label{sec: BCD estimated}
The objective function in the min-max problem (\ref{eq: BCD equation}), evaluated by data sample, can be viewed as a one-step importance sampling estimation. As shown in \cite{hanna2019importance}, importance sampling with estimated behavior policy has smaller MSE. Motivated by this fact and the heuristic that better objective function evaluation will lead to more accurate solution, we show that the BCH method can also be improved by using estimated behavior policy and obtain smaller asymptotic MSE. We will use this result to build theoretic guarantee for the performance of EMP method in Section \ref{sec: mis}. 

To formally state the theoretic result, we need introduce more notation. Assume that we are given a class of stationary distribution correction $\Omega =\{\omega(\eta;s):\eta\in \mathcal{E}_\eta\}$, and there exists $\eta_0\in \mathcal{E}_\eta$ such that the true distribution correction $\omega(s) = \omega(\eta_0;s) $. Let $\omega(\tilde{\eta};s)$ be the stationary distribution correction learned by the min-max problem (\ref{eq: BCD equation}) and $\omega(\hat{\eta};s)$ be that learned by a min-max problem using estimated policy:
\begin{equation}\label{eq: BCD estimated}
\min_{\omega}\max_f~ \mathbb{E}_{(s,a,s')\sim\mathcal{D}}\left[\left(\omega(s')-\omega(s)\frac{\pi(a|s)}{\hat{\pi}_0(a|s)}\right)f(s')\right]^2.
\end{equation}
The intuition is that the value of $\hat{\pi}_0$ is estimated from the data sample and appears in the denominator, as a result, it could cancel out a certain amount of random error in data sample. We use a maximum likelihood method to estimate the behavior policies for discrete and continuous control tasks. The details are in Appendix~\ref{appdix:esti}. Based on the proof techniques in \cite{henmi2007}, we establish the following theoretic guarantee that using estimated behavior policy yields better estimates of the stationary distribution correction.
\begin{thm}\label{thm: BCD estimated}
	Under some mild conditions, we have, asymptotically 
	$$E[(\hat{\eta}-\eta_0)^2]\leq E[(\tilde{\eta}-\eta_0)^2].$$ 
\end{thm}
As a direct consequence, we derive the finite-sample error bound for $\hat{\eta}$.
\begin{corollary}\label{crll: error bound} (informal) Let $N$ be the number of $(s,a,s')$ tuples in the data,
	$$\mathbb{E}[(\hat{\eta}-\eta_0)^2]=O\left(\frac{1}{N}\right)$$.
\end{corollary}
The precise conditions for Theorem \ref{thm: BCD estimated} and Corollary \ref{crll: error bound} to hold and their proofs are in Appendix \ref{appdx: thm1}. 
\section{EMP for Multiple Behavior Policies}\label{sec: mis}
In this section, we shall propose our EMP method for off-policy policy evaluation with multiple known or unknown behavior policies and establish theoretic results on variance reduction of EMP.

Before that, we first give a detailed description on the data sample and its distribution. Assume the state-action-next-state tuples are generated by $m$ different unknown behavior policies $\pi_j, j= 1,2,...,m$. Let $d_{\pi_j}(s)$ be the stationary state distribution and $N_j$ be the number of state-action-next-state tuples by policy $\pi_j$, for $j=1,2,...,m$. Let $N=\sum_j N_j$ and denote by $w_j=N_j/ N$ the proportion of data generated by policy $\pi_j$. We use $\mathcal{D}$ to denote the data set and $\mathcal{D}=\{(s_{j,n_j}, a_{j,n_j}, s'_{j,n_j} ):j=1,2,..,m, n_j = 1,2,...,N_j\}$. Note that the policy label $j$ in the subscript is only for notation clarity and it is not revealed in the data. Then, a single $(s,a,s')$ tuple simply follows the marginal distribution $d_0(s,a,s'):= \sum_j w_jd_{\pi_j}(s)\pi_j(a|s)T(s'|a,s)$. With slight notation abusion, we write $d_0(s,a) = \sum_j w_j d_{\pi_j}(s)\pi(a|s)$. 


\subsection{EMP Method}\label{sec: mis method}
Now we derive the EMP method in the multiple-behavior-policy setting and explicitly explain what is the mixed policy to be estimated in EMP.  

Let $d_0:= \sum_j w_jd_{\pi_j}$ be the mixture of stationary distributions of the behavior policies. For each state-action pair $(a,s)$, define $\pi_0(a|s)$ as the weighted average of the behavior policies:
\begin{equation}\label{eq: pi_0}
\pi_0(a|s):=\sum_{j} \frac{w_jd_{\pi_j}(s)}{d_{\pi_0}(s)}\pi_j(a|s), \forall~ (s,a).
\end{equation}
It is easy to check that for each $s$, $\pi_0(\cdot|s)$ is a distribution on the action space and hence defines a policy by itself. We call $\pi_0$ the \textit{mixed policy}. Let $\omega = d_\pi/d_0$, which is a state distribution ratio. Then, $d_0$, $\pi_0$ and $\omega$ satisfy the following relation with the average reward $R_{\pi}$. 

\begin{proposition}\label{prop: emp reward}
	\begin{equation}\label{eq: emp reward}
	\begin{aligned}
	R_\pi = E_{(s,a)\sim d_0}\left[\omega(s) \frac{\pi(a|s)}{\pi_0(a|s)}r(s,a)\right].
	\end{aligned}
	\end{equation}
\end{proposition}
Besides, the state distribution ratio $\omega$ can be characterized by the stationary equation.
\begin{proposition}\label{prop: emp omega}
	The function $\omega(s)=d_\pi(s)/d_{\pi_0}(s)$ (up to a constant) if and only if, 
	\begin{equation}\label{eq: emp omega}
	\mathbb{E}_{(s,a,s')\sim d_0}\left[\left(\omega(s')-\omega(s)\frac{\pi(a|s)}{\pi_0(a|s)}\right)f(s')\right]=0, \text{ for all }f: S\to\mathbb{R}.
	\end{equation}
\end{proposition}
In the special case when $m=1$, i.e. the data are generated by a single behavior policy, Proposition \ref{prop: emp omega} reduces to Theorem 1 of \cite{liumethod}. The above two Propositions indicate that, to certain extend, the $(s,a,s')$ tuples generated by multiple behavior policies can be pooled together and treated as if they are generated by a single behavior policy $\pi_0$.

Note that expression of $\pi_0$ (\ref{eq: pi_0}) involves not only the behavior policies but also the state stationary distributions. In EMP method, we shall use a pre-estimation step to generate an estimate $\hat{\pi}_0$ from the data. Based on Proposition \ref{prop: emp omega}, the state distribution ratio $\omega$ can be estimated by the following min-max problem
\begin{equation}\label{eq: multiple unknown}
\min_{\omega}\max_f~ \mathbb{E}_{(s,a,s')\sim \mathcal{D}}\left[ \left(\omega(s')-\omega(s)\frac{\pi(a|s)}{\hat{\pi}_0(a|s)}\right)f(s')\right]^2.
\end{equation}
Finally, EMP estimates the average reward according to (\ref{eq: emp reward}) where $d_0$ is approximate by data $\mathcal{D}$.

Applying Theorem \ref{thm: BCD estimated}, we show that using the estimated $\hat{\pi}_0$ in EMP can actually reduce the MSE of the learned stationary distribution ratio $\omega$.
\begin{proposition}\label{thm: multiple estimated}
	Under the same conditions of Theorem \ref{thm: BCD estimated},
	if $\omega(\tilde{\eta};s)$ and $\omega(\hat{\eta};s)$ are the stationary distribution correction learned from (\ref{eq: multiple unknown}) and from the same min-max problem but with exact value of $\pi_0$, then, asymptotically
	$$E[(\hat{\eta}-\eta_0)^2]\leq E[(\tilde{\eta}-\eta_0)^2].$$ 
	As a result, $E[(\hat{\eta}-\eta_0)^2] = O\left(\frac{1}{N}\right)$.
\end{proposition}

\subsection{Why Pooling is Beneficial for EMP}\label{sec: mis analysis}
One important feature of EMP is that it pools the data from different policy behaviors together and treat them as if they are from a single mixed policy. Of course, pooling makes EMP applicable to settings with minimal information on the behavior policies, for instance, EMP does not even require the knowledge on the number of behavior policies. In this part, we show that, the pooling feature of EMP is not just a compromise to the lack of behavior policy information, it also leads to variance reduction in an intrinsic manner.

If instead, the data can be classified according to the behavior policies and treated separately, we can still use EMP, which reduces to (\ref{eq: BCD estimated}), or any other single-behavior-policy method, to obtain the stationary distribution correction $\omega_j=d_{\pi}/d_{\pi_j}$ for each behavior policy. Given $\omega_j$, a common approach for variance reduction is to apply multiple importance sampling (MIS) \citep{Tirinzoni2019,Veach1995} technique and the average reward estimator is of the form
\begin{equation}\label{eq: MIS}
\hat{R}_{MIS}=\sum_{j=1}^m\frac{1}{N_j}\sum_{n=1}^{N_j}h_j(s_{j,n})\omega_j(s_{j,n})\pi(a_{j,n}|s_{j,n})r(s_{j,n},a_{j,n}),
\end{equation}
where the function $h$ is often referred to as heuristics and must be a partition of unity, i.e., $\sum_j h_j(s)=1$ for all $s\in S$. It has been proved by \citep{Veach1995} that MIS is unbiased, and, for given $w_j = N_j/N$, there is an optimal heuristic function to minimize the variance of $\hat{R}_{MIS}$.



\begin{proposition}\label{prop: heuristic}
	For MIS with fixed values of $w_j, j=1,2,...,m$, among all possible values of heuristics $h$, the balanced heuristic 
	$$h_j(s)= \frac{w_j d_{\pi_j}(s)}{\sum_{j=1}^m w_j d_{\pi_j}(s)}, ~\forall j = 1,2,...,m\text{ and }s\in S,$$
	reaches the minimal variance.
\end{proposition}
Plug the optimal heuristic $h_j(s)$ into MIS estimator (\ref{eq: MIS}), and we will obtain that  the optimal MIS estimator coincides with the EMP estimator (\ref{eq: emp reward}), i.e.
\begin{equation}\label{eq: omega}
\begin{aligned}
\hat{R}_{MIS}
& = \mathbb{E}_{(s,a)\sim \mathcal{D}}\left[\frac{d_{\pi}(s)}{d_0(s)}\pi(a|s)r(s,a)\right].
\end{aligned}
\end{equation}
In this light, by pooling the data together and directly learning $\omega$. EMP also learns the optimal MIS weight inexplicitly.


\begin{figure*}[h]
 \centering
  \subfigure[State Distribution Comparison]{
 	\includegraphics[ width=1.6in]{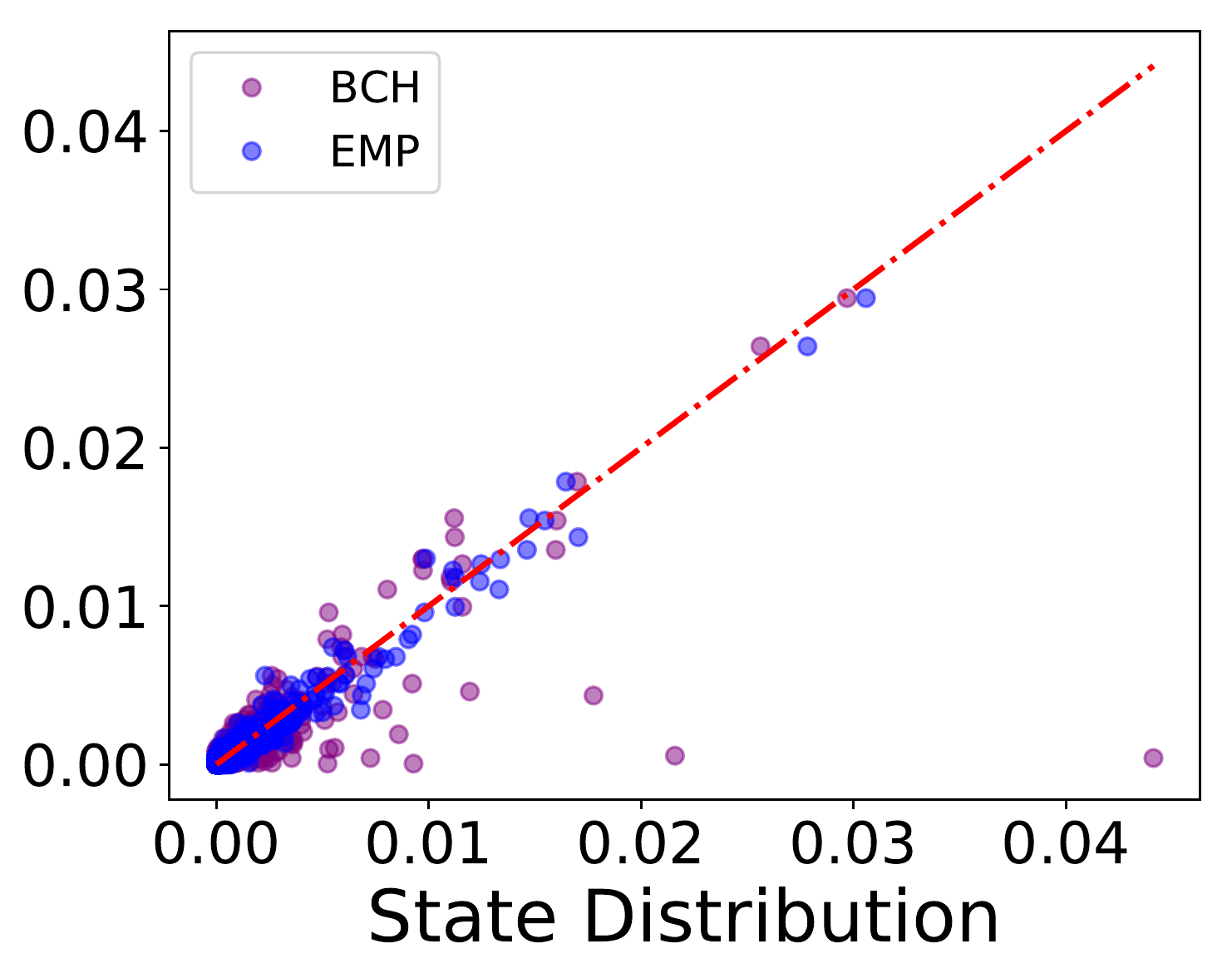}
 	\label{fig:tv-subfigure3}}
 ~
 \subfigure[TV Distance]{%
 	\includegraphics[ width=1.6in]{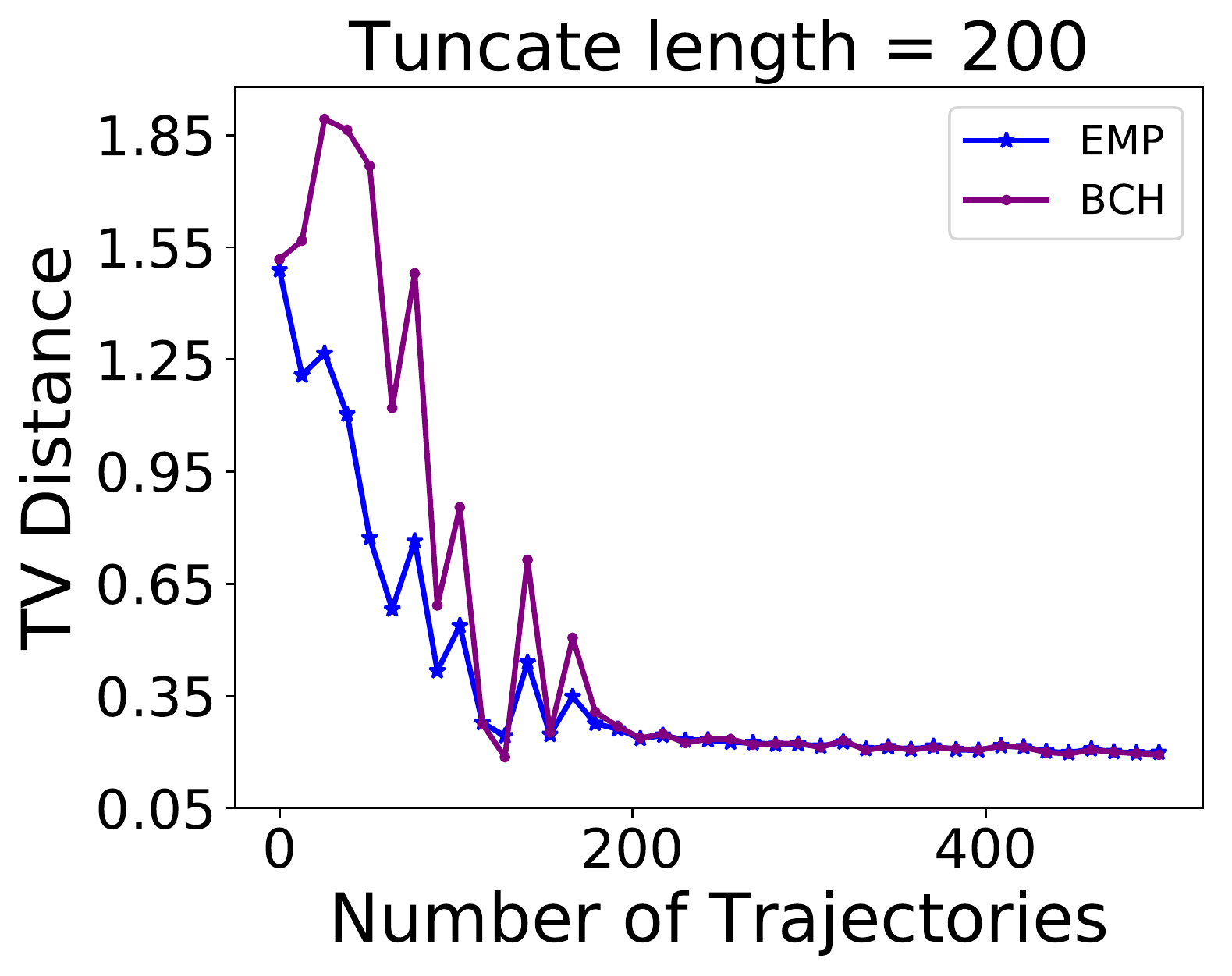}
 	\label{fig:tv-subfigure1}
 	\includegraphics[ width=1.6in]{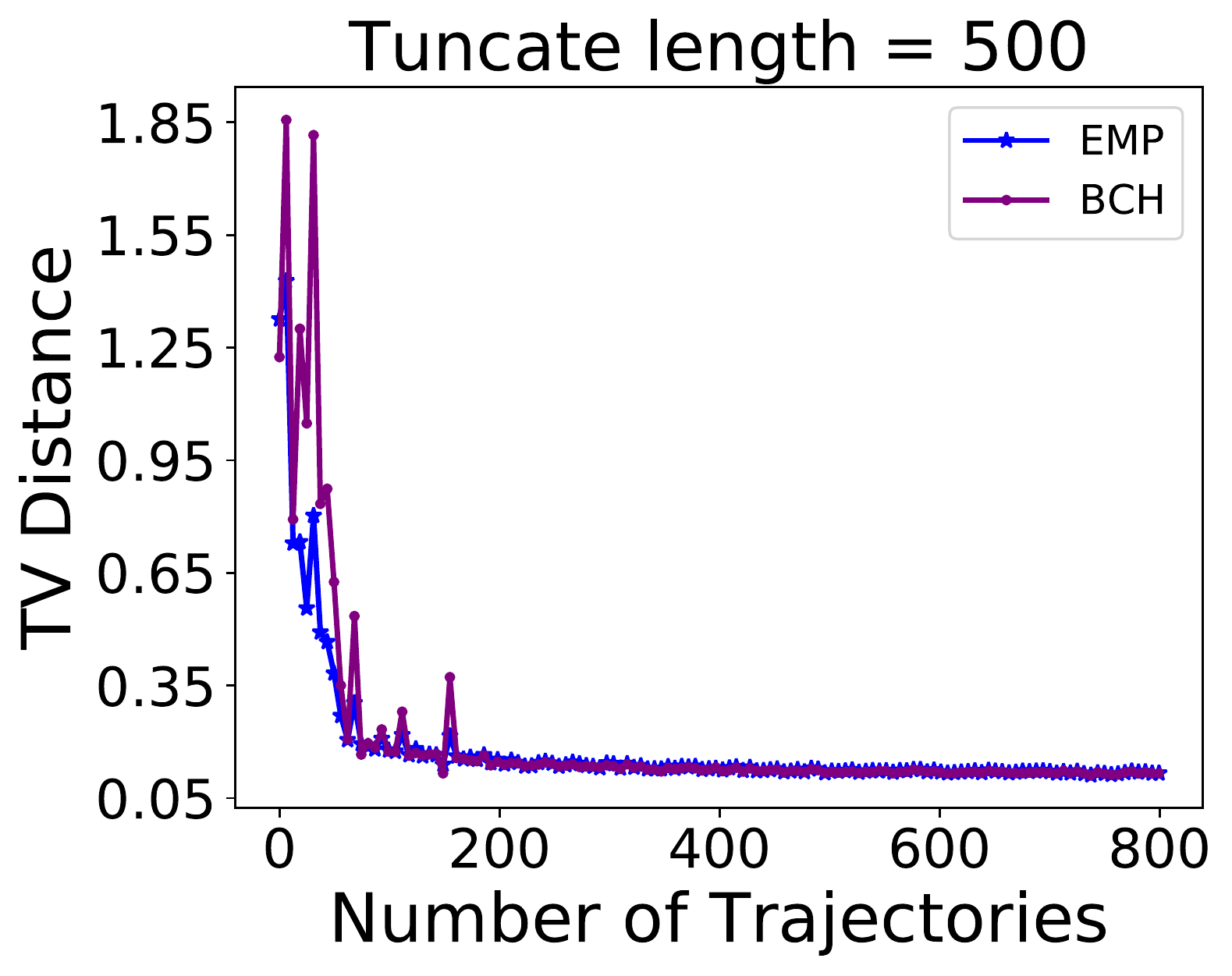}
 	\label{fig:tv-subfigure2}}
 ~
 \vspace{-1em}
 \caption{(a) shows that scatter plot of pairs ($\hat{d}_{\pi_\text{true}}$, $d_\pi$) and pairs ($\hat{d}_{\pi_\text{esti}}$, $d_\pi$). The diagonal line indicates exact estimation. The default values of the number of trajectories is 200, and the length of horizon is 200. (b) shows the weighted total variation distance (TV distance) between $\hat{d}_{\pi_\text{true}}$ and $d_\pi$, $\hat{d}_{\pi_\text{esti}}$ and $d_\pi$ respectively, along
different number of trajectories and the length of horizons. }

\label{dis-com}
\end{figure*}

\begin{figure*}
	\centering
	\hspace{-1.2em}
	\subfigure[Taxi]{%
		\includegraphics[ width=1.24in]{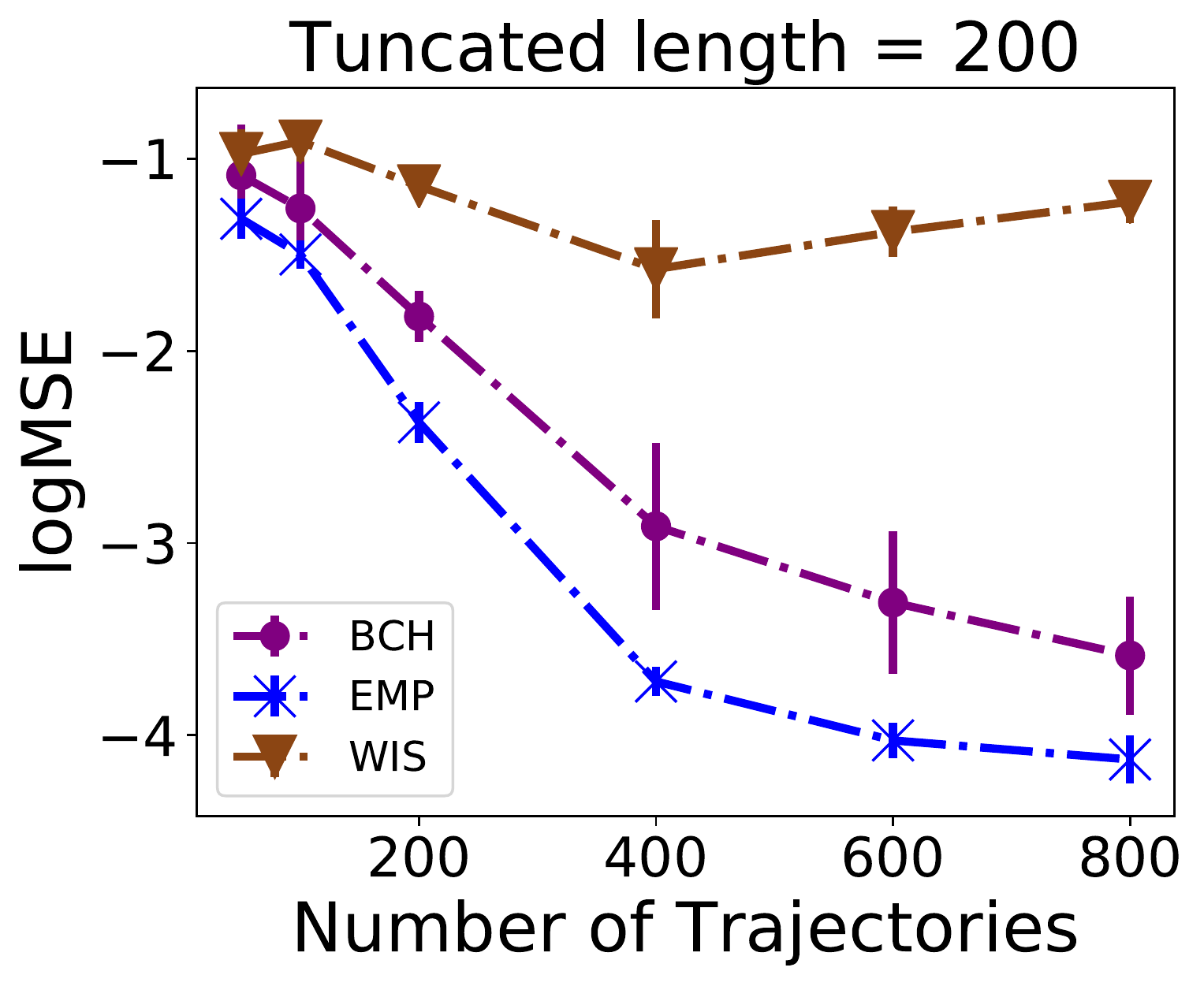}
		\label{fig:subfigure5}
		\includegraphics[ width=1.24in]{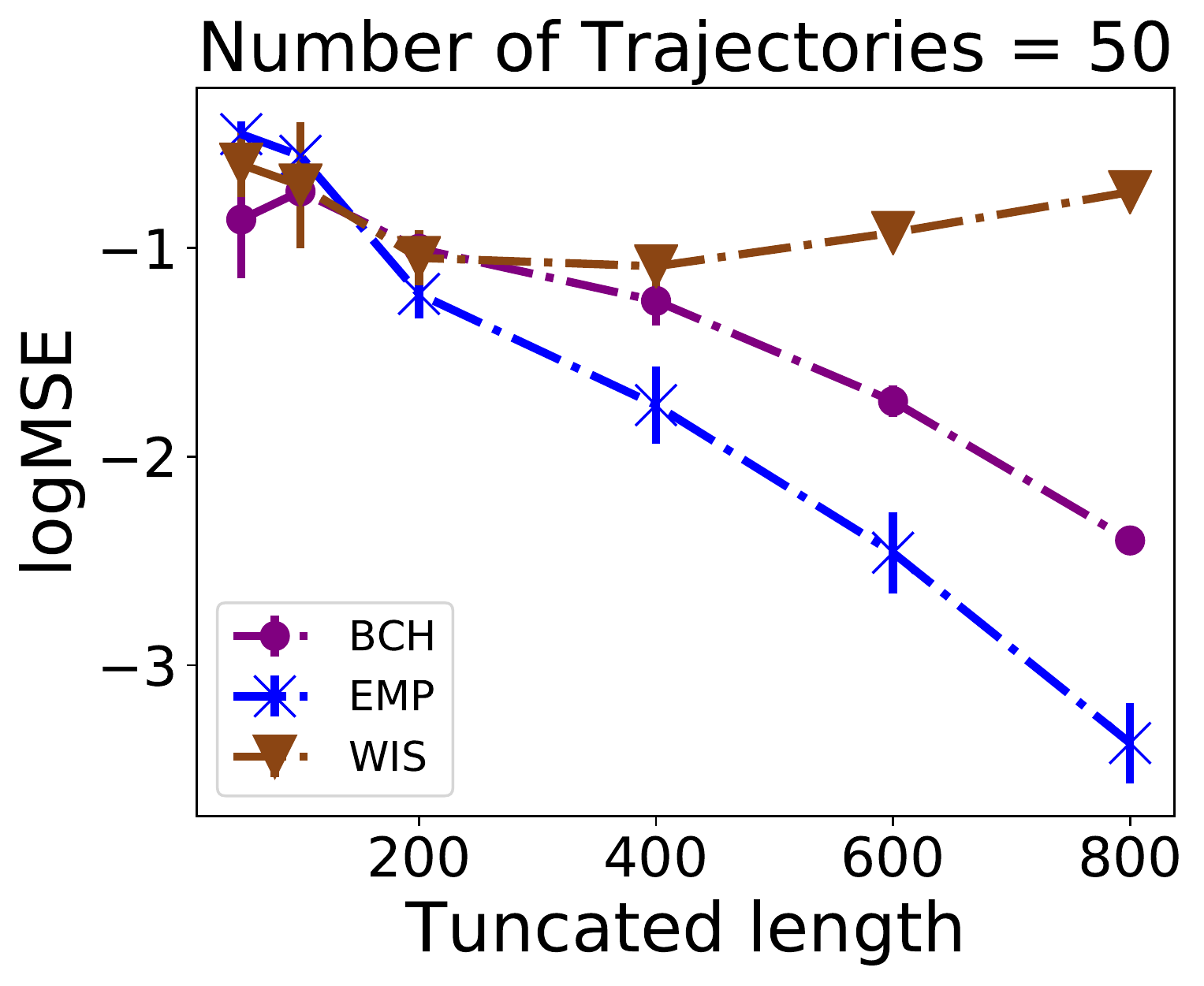}
		\label{fig:subfigure5}}
~
	\subfigure[Singlepath]{%
		\includegraphics[width=1.24in]{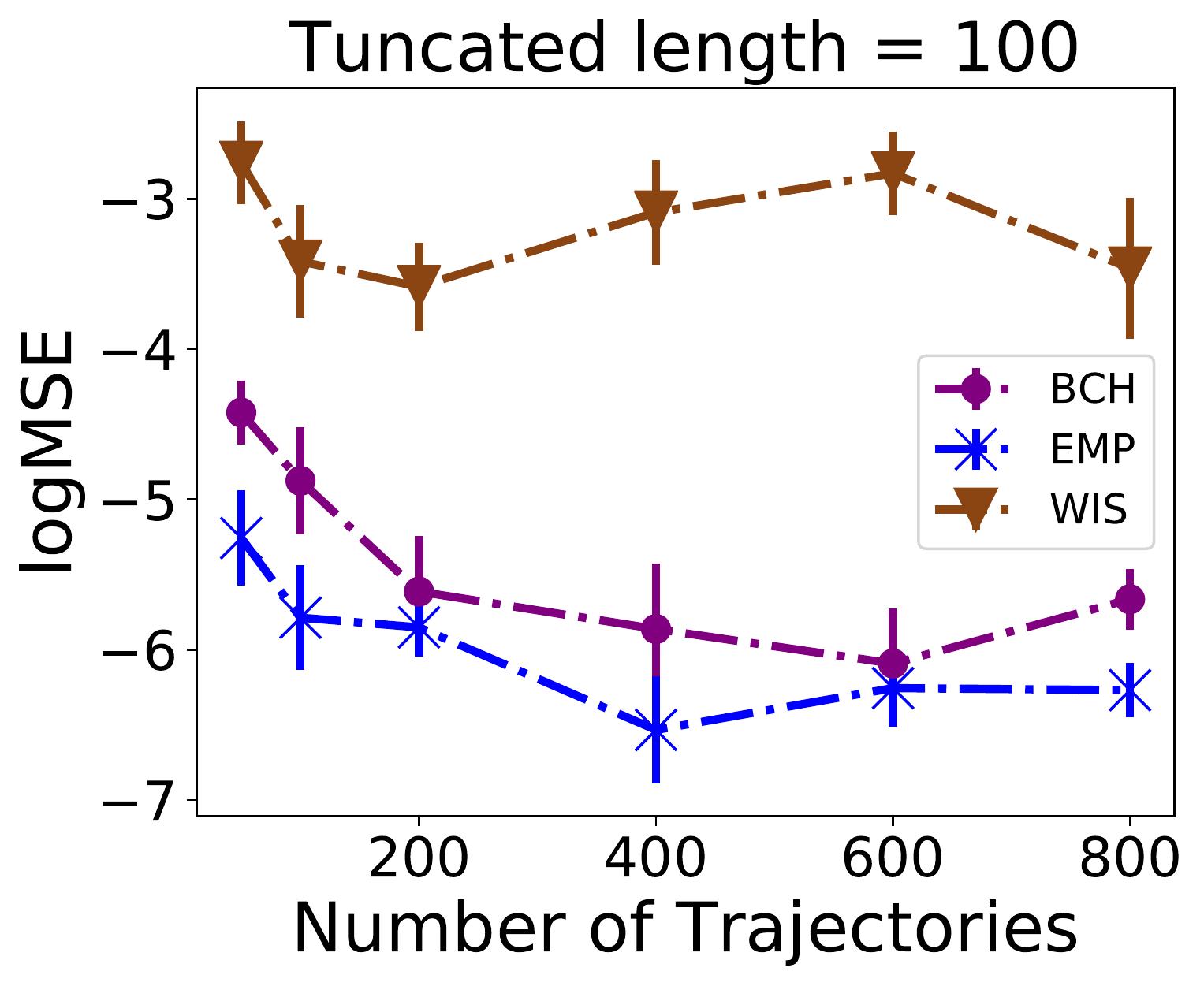}
		\label{fig:subf1}
		\includegraphics[ width=1.24in]{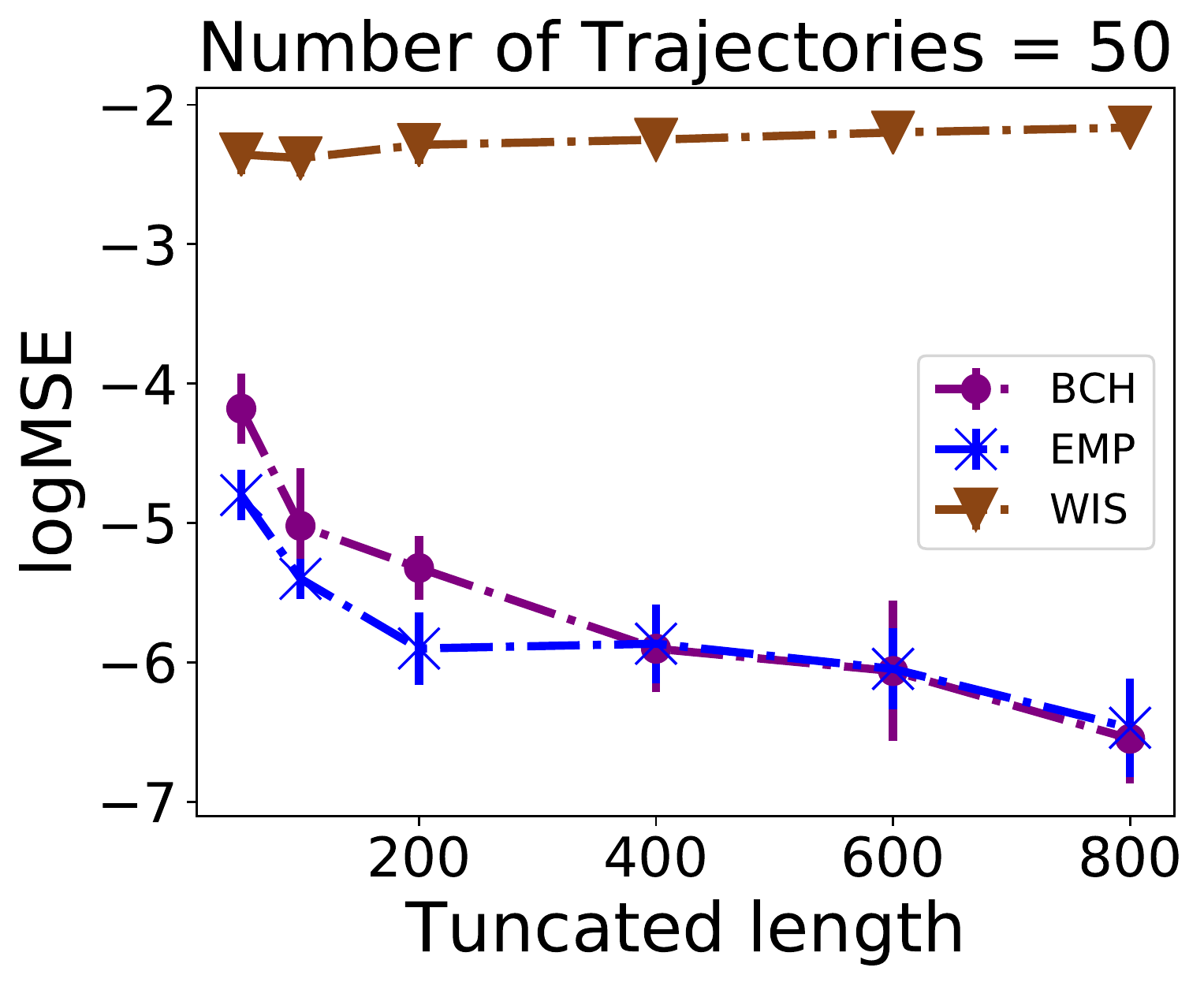}
		\label{fig:subfigure3}}

	\vspace{-.12in} \subfigure[Gridworld]{%
		\includegraphics[ width=1.24in]{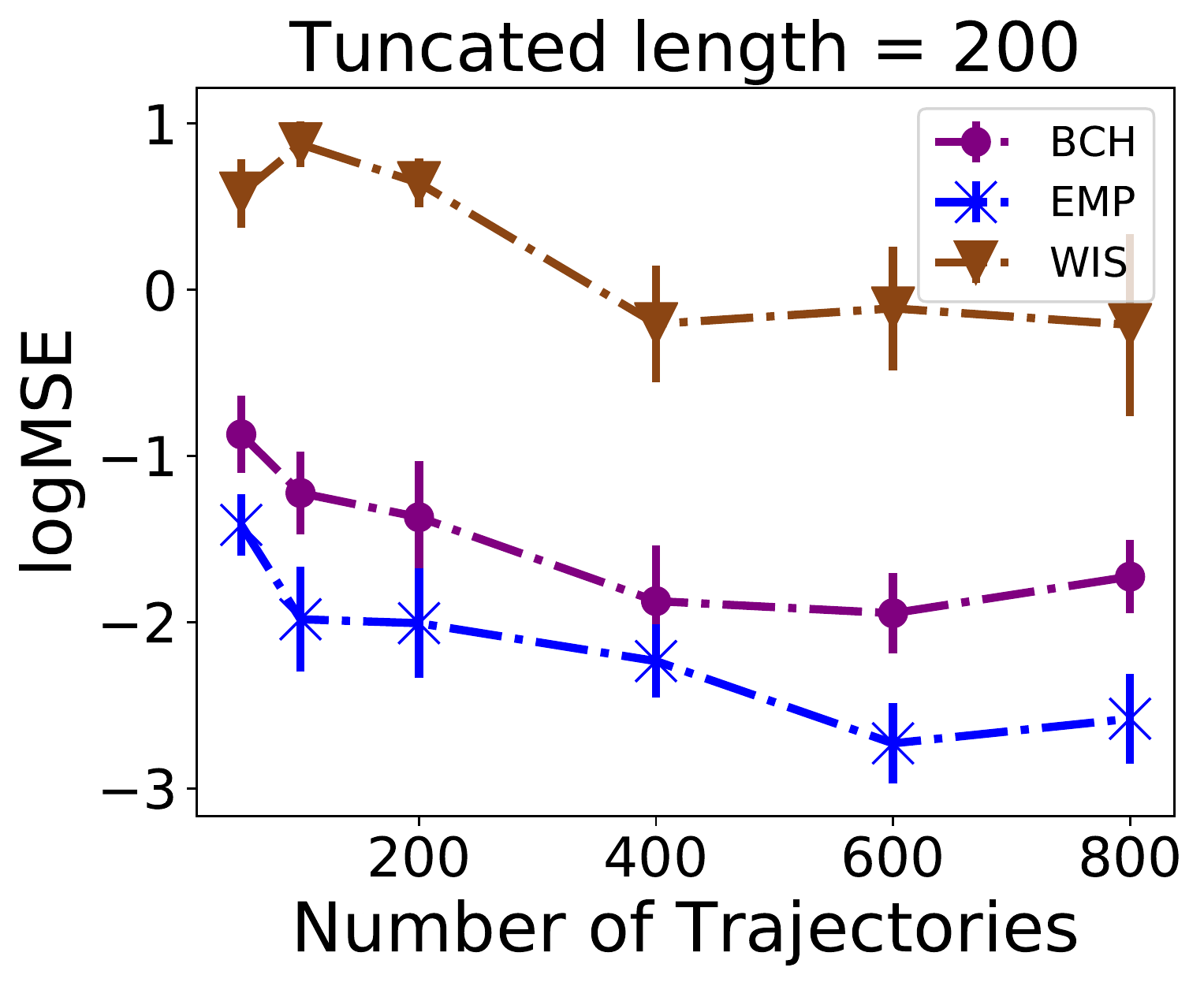}
		\label{fig:subfigure6}
		\includegraphics[ width=1.24in]{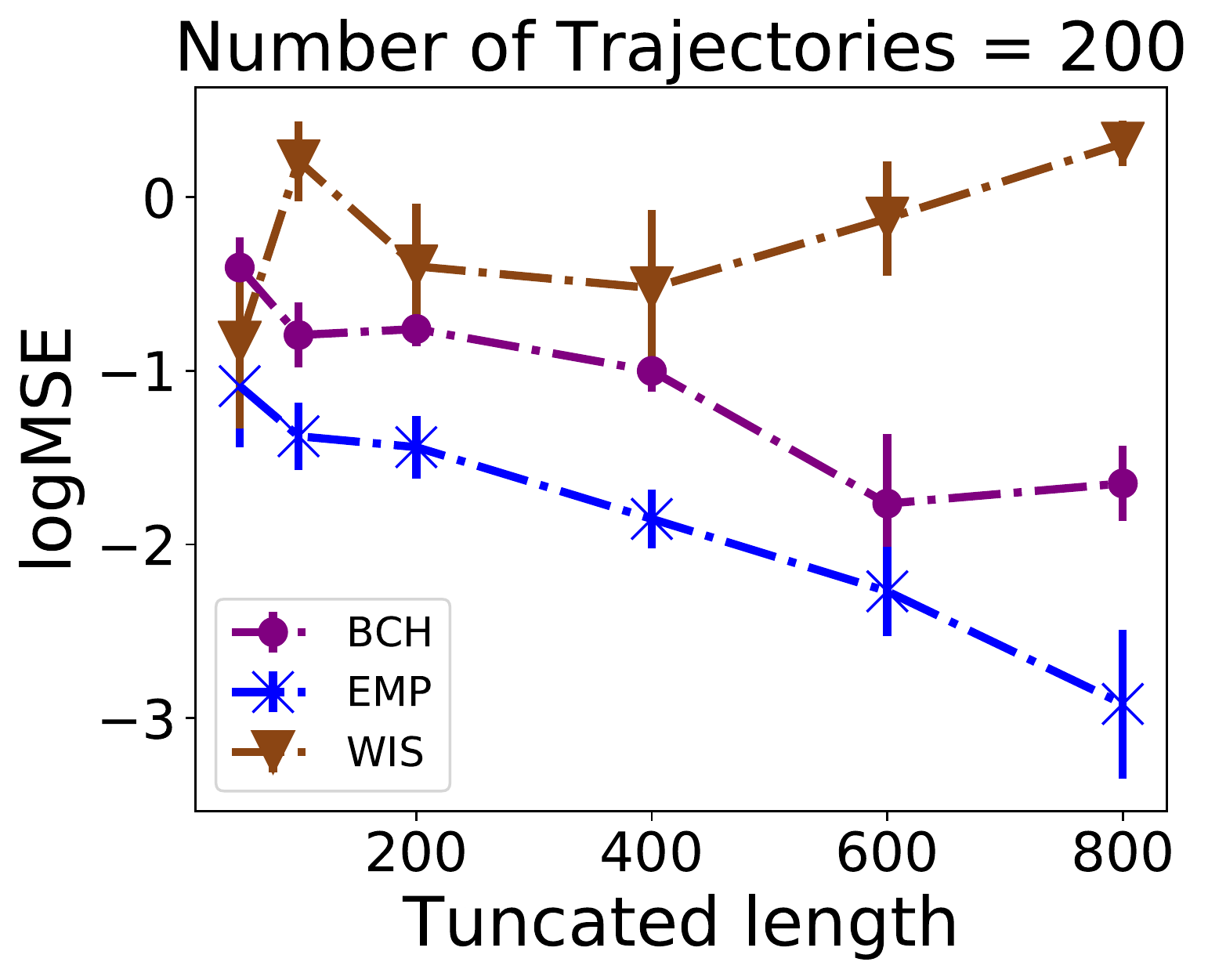}
		\label{fig:subfigure6}}
	\vspace{.08in}
	\subfigure[Pendulum]{%
		\includegraphics[width=1.3in]{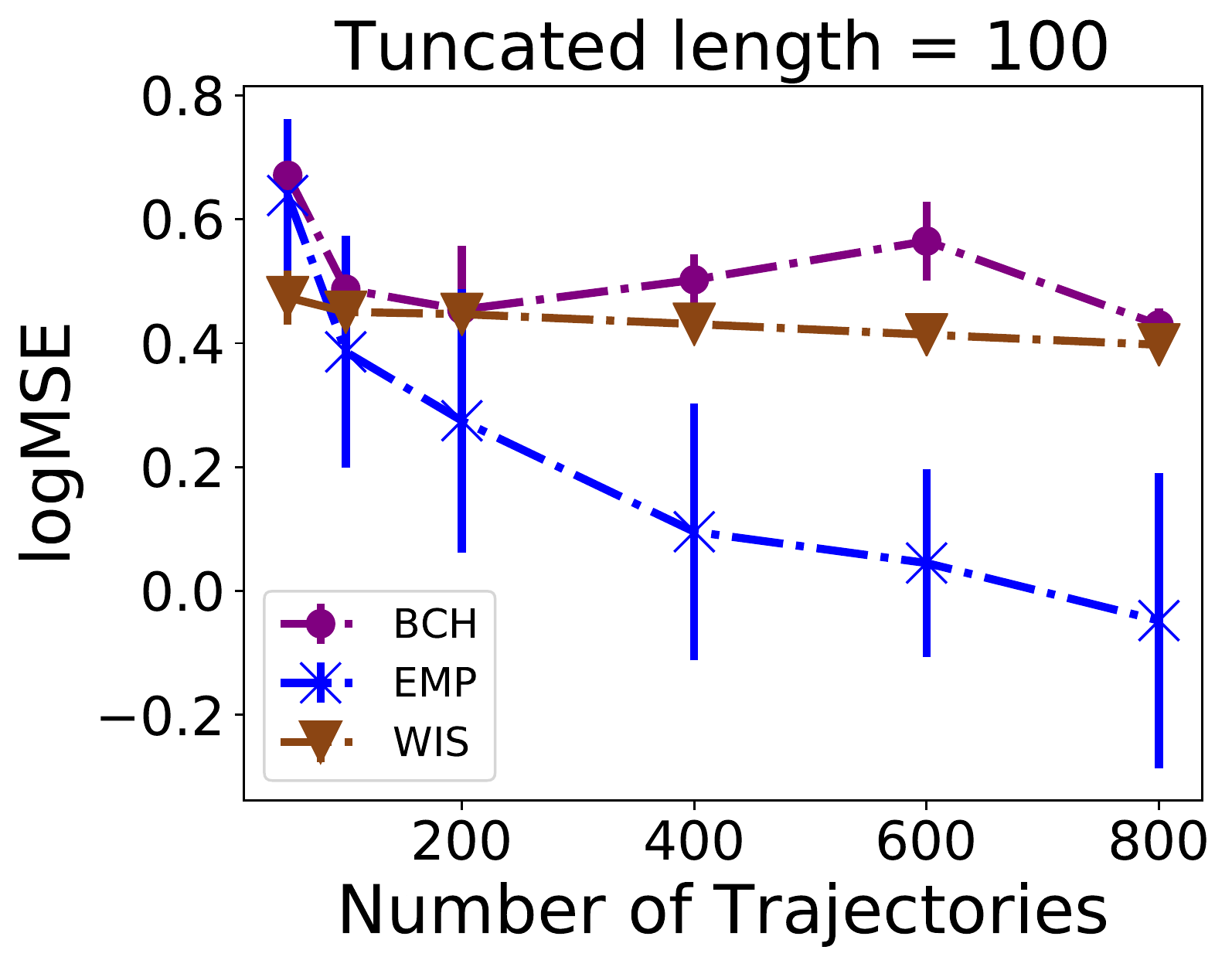}
		\label{fig:subfigure2}%
		\includegraphics[width=1.3in]{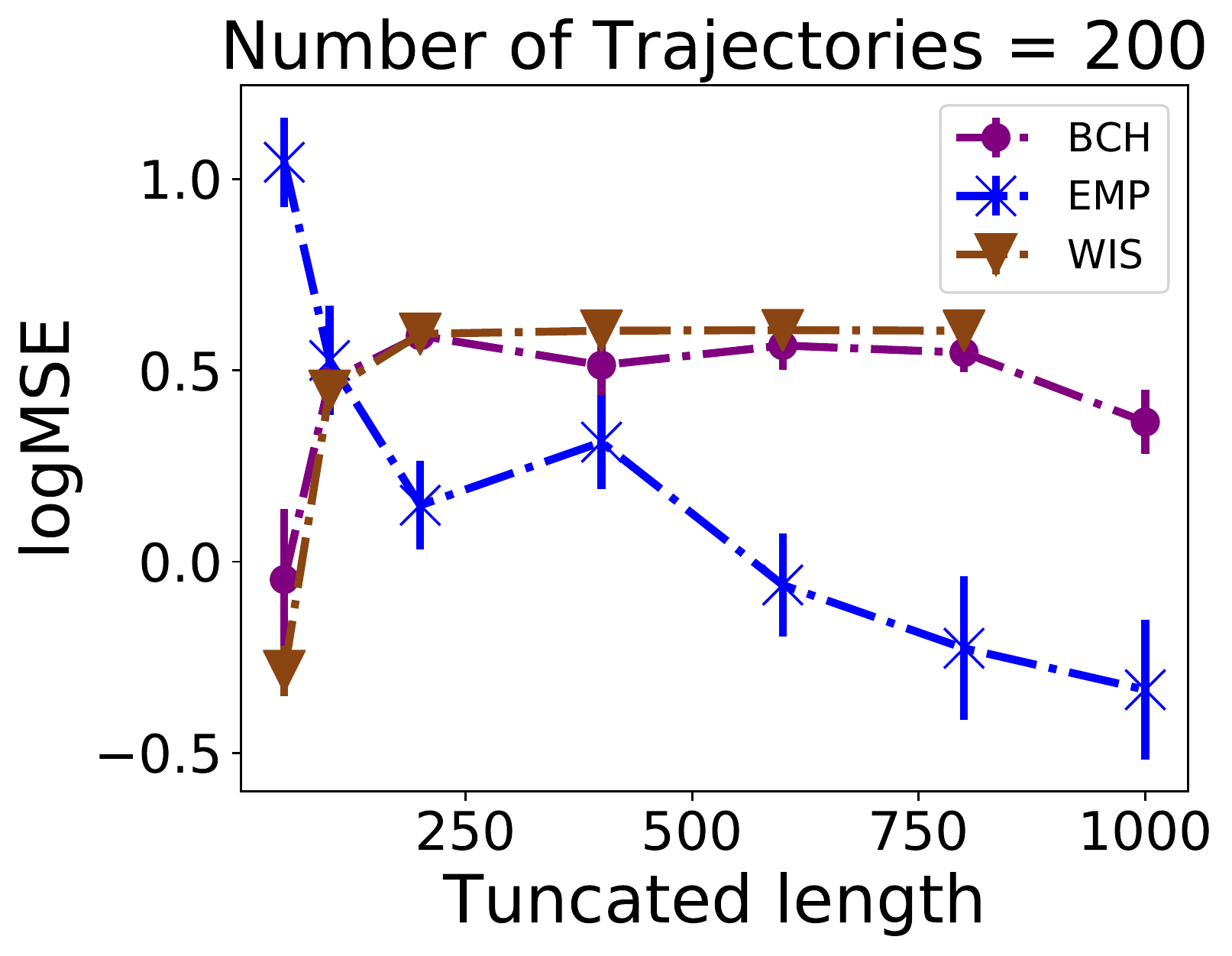}
		\label{fig:subfigure4}}
~
\vspace{-1em}
\caption{\label{single-com} Single-behavior-policy results of BCH, EMP and WIS across continuous and discrete environments with average reward. Each node indicates the mean value and the bars represents the standard error of the mean.}
\end{figure*}

\section{Experiment}\label{sec: experiment}

In this section, we conduct experiments in three discrete-control tasks Taxi, Singlepath, Gridworld and one continuous-control task Pendulum (see Appendix~\ref{app:env} for the details), with following purposes: (i) to compare the performance of distribution correction learning using policy-aware, policy-agnostic and partially agnostic-policy methods (in Sec~\ref{sec: experiment single}); (ii) to compare the performance of the proposed EMP with existing OPPE methods (in Sec~\ref{sec: experiment single} and~\ref{sec: experiment multiple}); (iii) to explore potential improvement of EMP methods (in Sec~\ref{sec: experiment multiple}). We will release the codes with the  publication of this paper for relevant study.

\subsection{Results for Single Behavior Policy}\label{sec: experiment single}
\vspace{-0.5em}
In this section, we compare the EMP method with the BCH method and step-wise importance sampling (IS) in the setting of single-behavior policy, i.e. the data is generated from a single behavior policy.

\textbf{Experiment Set-up.} 
A single behavior policy which is learned by a certain reinforcement learning algorithm~\footnote{We use Q-learning in discrete control tasks and Actor Critic in continuous control tasks.} for evaluating BCH and IS. This single behavior policy then generates a set of trajectories consisting of s-a-s-r tuples. These tuples are used to estimate the behaviour policy for EMP methods as well as estimating the stationary distribution corrections for estimating the average step reward of the target policy. 
\begin{figure*}[h]
 \centering
  \subfigure[Taxi]{%
 	\includegraphics[ width=1.22in]{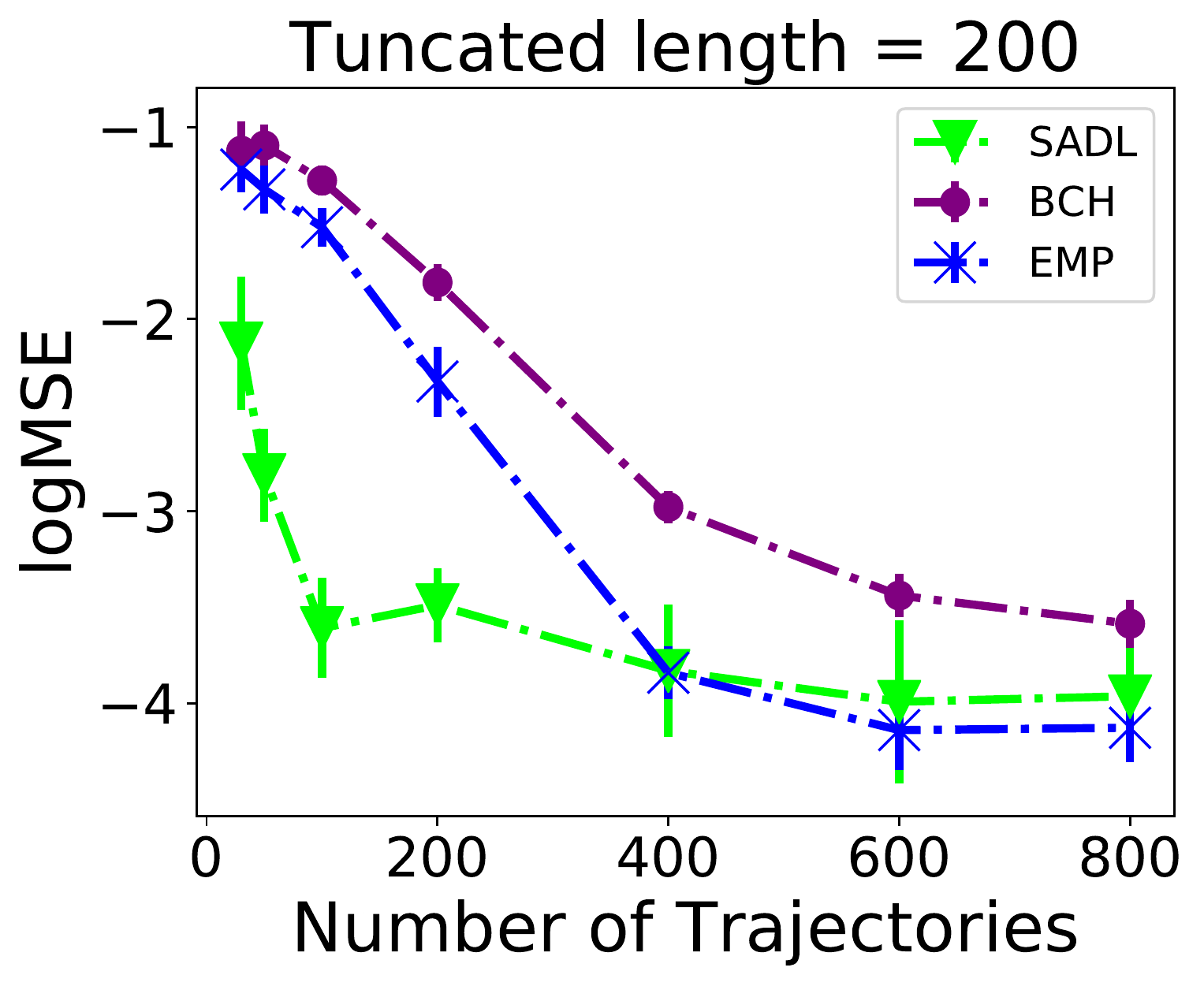}
 	\label{fig:subfigure3}
 	\includegraphics[ width=1.22in]{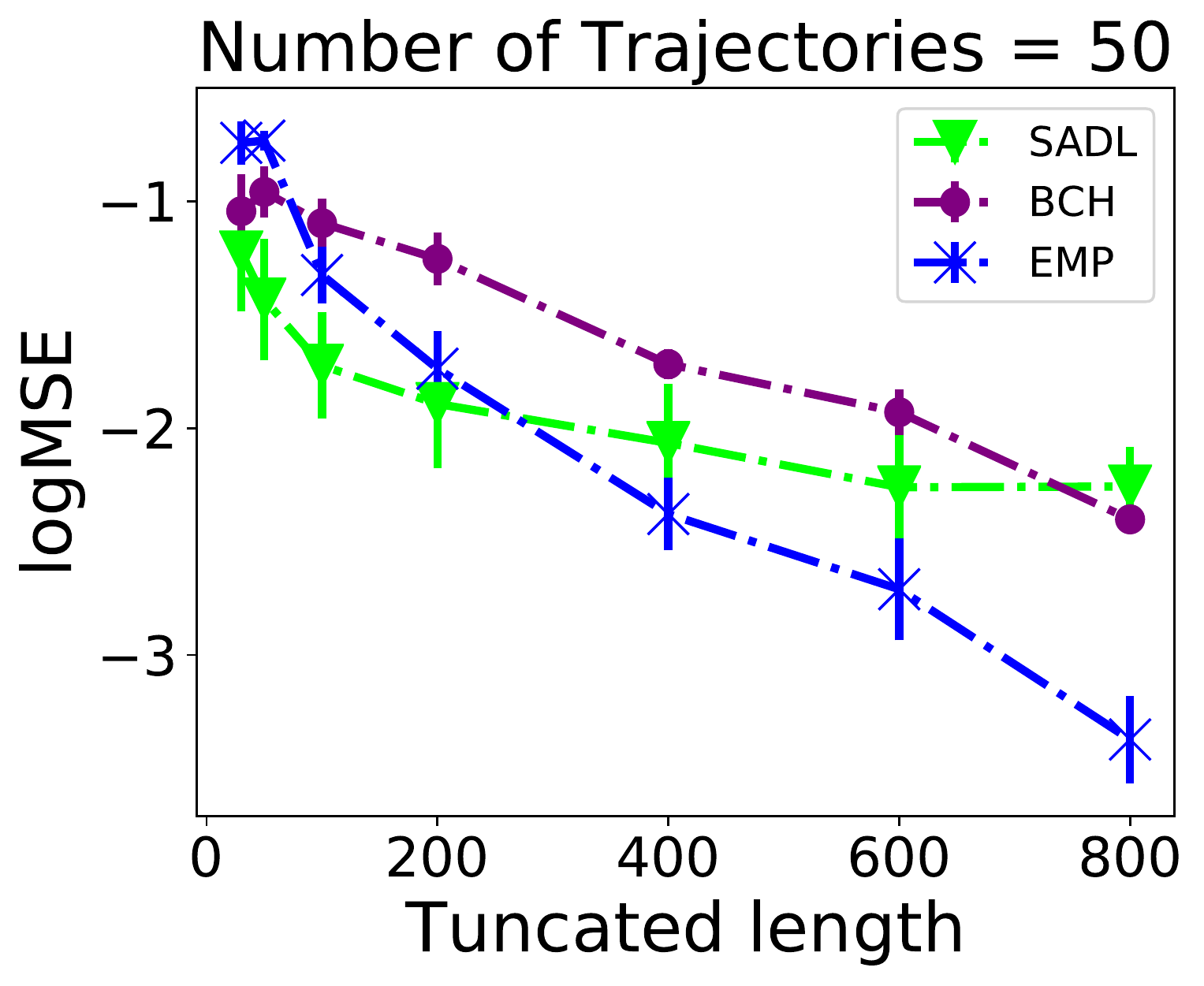}
 	\label{fig:subfigure1}}
 ~
   \subfigure[Pendulum]{%
 	\includegraphics[ width=1.24in]{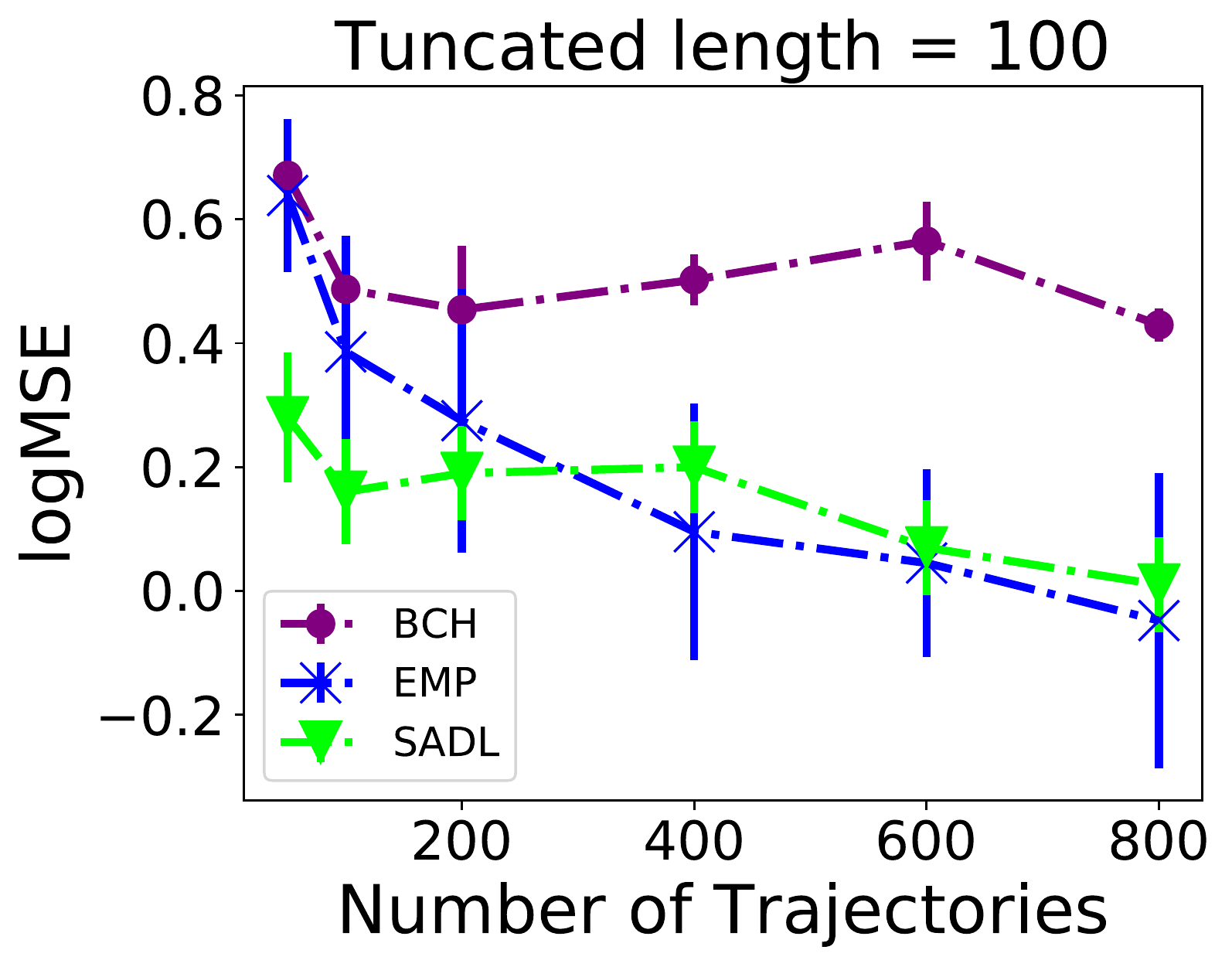}
 	\label{fig:subfigure3}
 	\includegraphics[ width=1.24in]{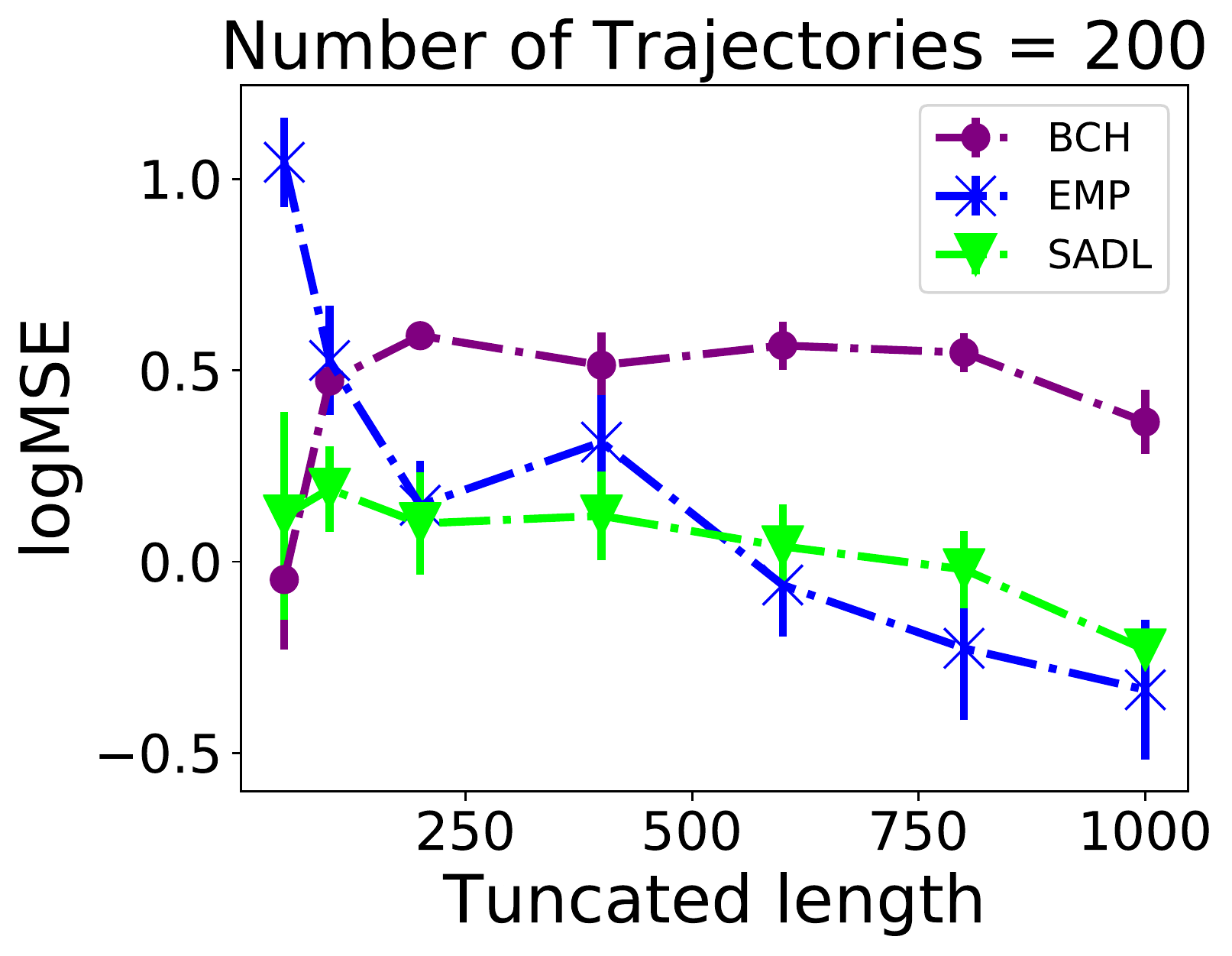}
 	\label{fig:subfigure1}}
 ~
 	\caption{Comparison results among policy-aware (BCH), partially policy-agnostic (EMP) and policy-agnostic (SADL) on continuous and discrete control tasks. }
 \label{scenario3}
\end{figure*}

\textbf{Stationary Distribution Learning Performance.}
We choose the Taxi domain as an example to compare the stationary distribution $\hat{d}_{\pi_\text{true}}$ and $\hat{d}_{\pi_\text{esti}}$ learned by BCH and EMP. Figure~\ref{fig:tv-subfigure3} shows the scatter pairs $(\hat{d}_{\pi_\text{true}},d_\pi)$ and $(\hat{d}_{\pi_\text{esti}},d_\pi)$ estimated by 200 trajectories of 200 steps.
It shows that $\hat{d}_{\pi_\text{esti}}$ approximate $d_\pi$ better than $\hat{d}_{\pi_\text{true}}$. Figure~\ref{fig:tv-subfigure1} and Figure~\ref{fig:tv-subfigure2} compare the TV distance from $\hat{d}_{\pi_\text{true}}$ and $\hat{d}_{\pi_\text{esti}}$ to $d_\pi$ under different data sample sizes. The results indicate that both $\hat{d}_{\pi_\text{true}}$ and $\hat{d}_{\pi_\text{esti}}$ converge, while $\hat{d}_{\pi_\text{esti}}$ converges faster and is significantly closer to $d_\pi$ when the data size is small. These observations are well consistent with Theorem~\ref{thm: BCD estimated}.

\begin{figure*}[h]
	\centering
	\hspace{-1.8em}
	\subfigure[Taxi]{%
		\includegraphics[ width=1.24in]{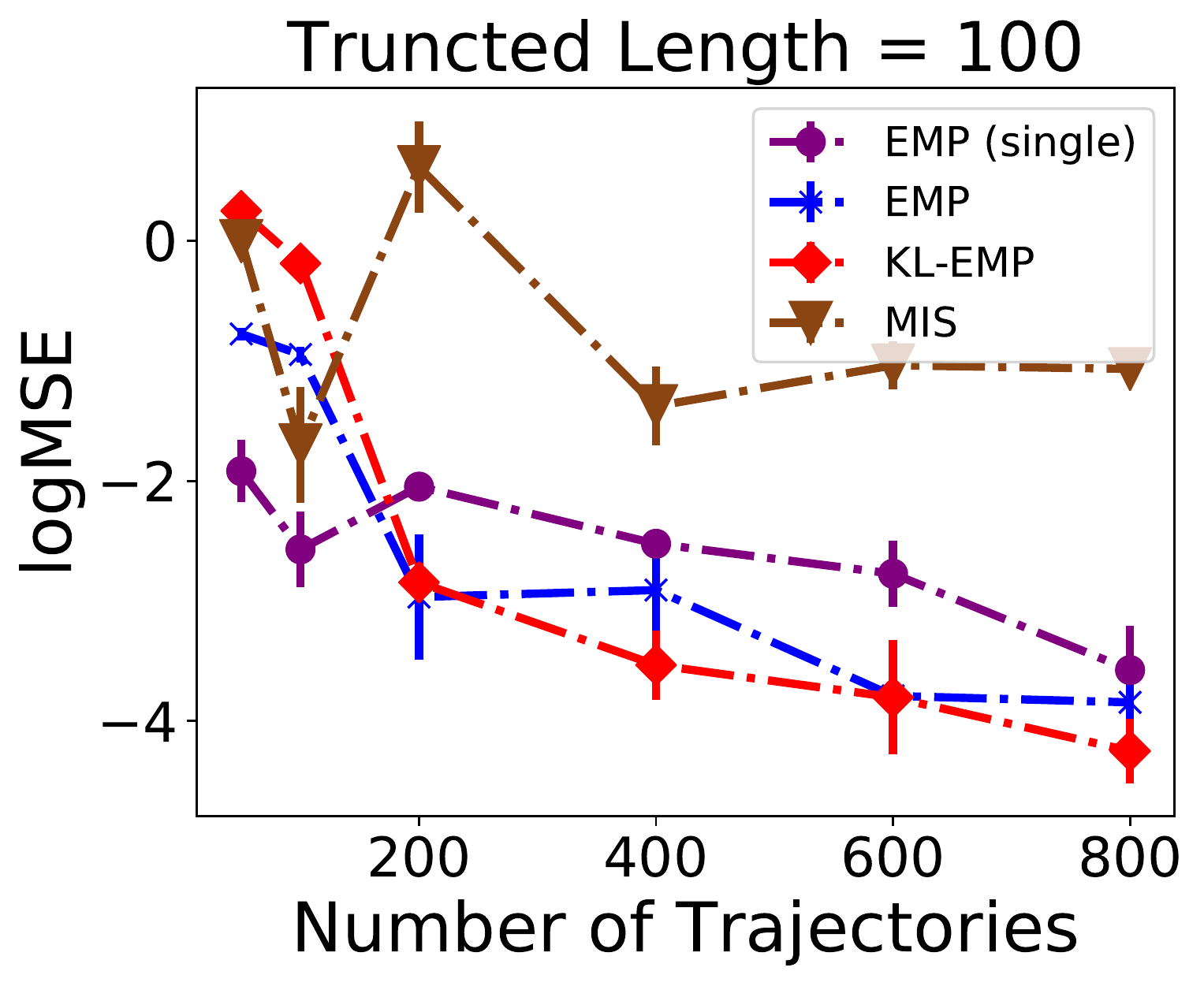}
		\label{fig:subfigure5}
		\includegraphics[ width=1.24in]{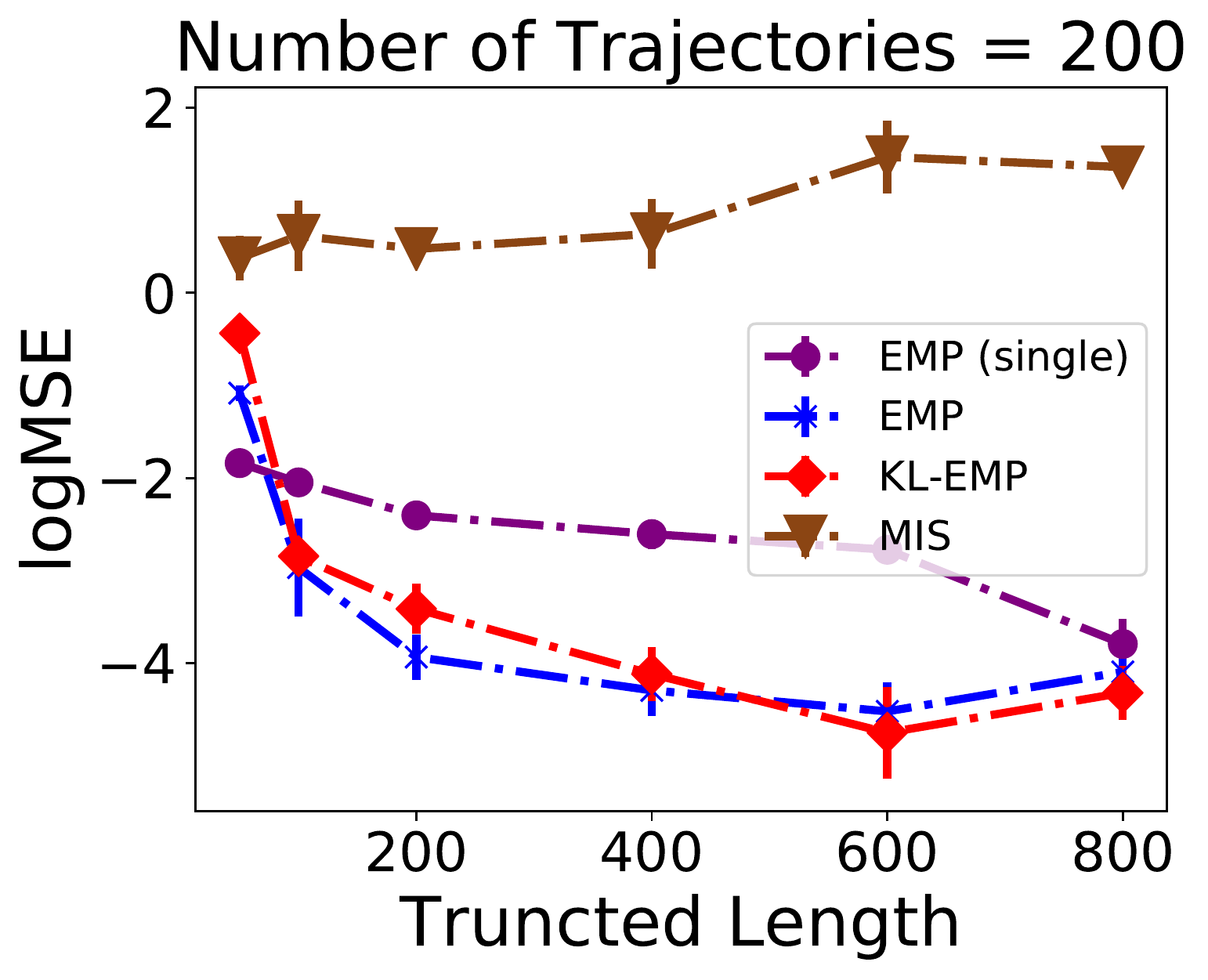}
		\label{fig:subfigure5}}
~
	\subfigure[Singlepath]{%
		\includegraphics[width=1.28in]{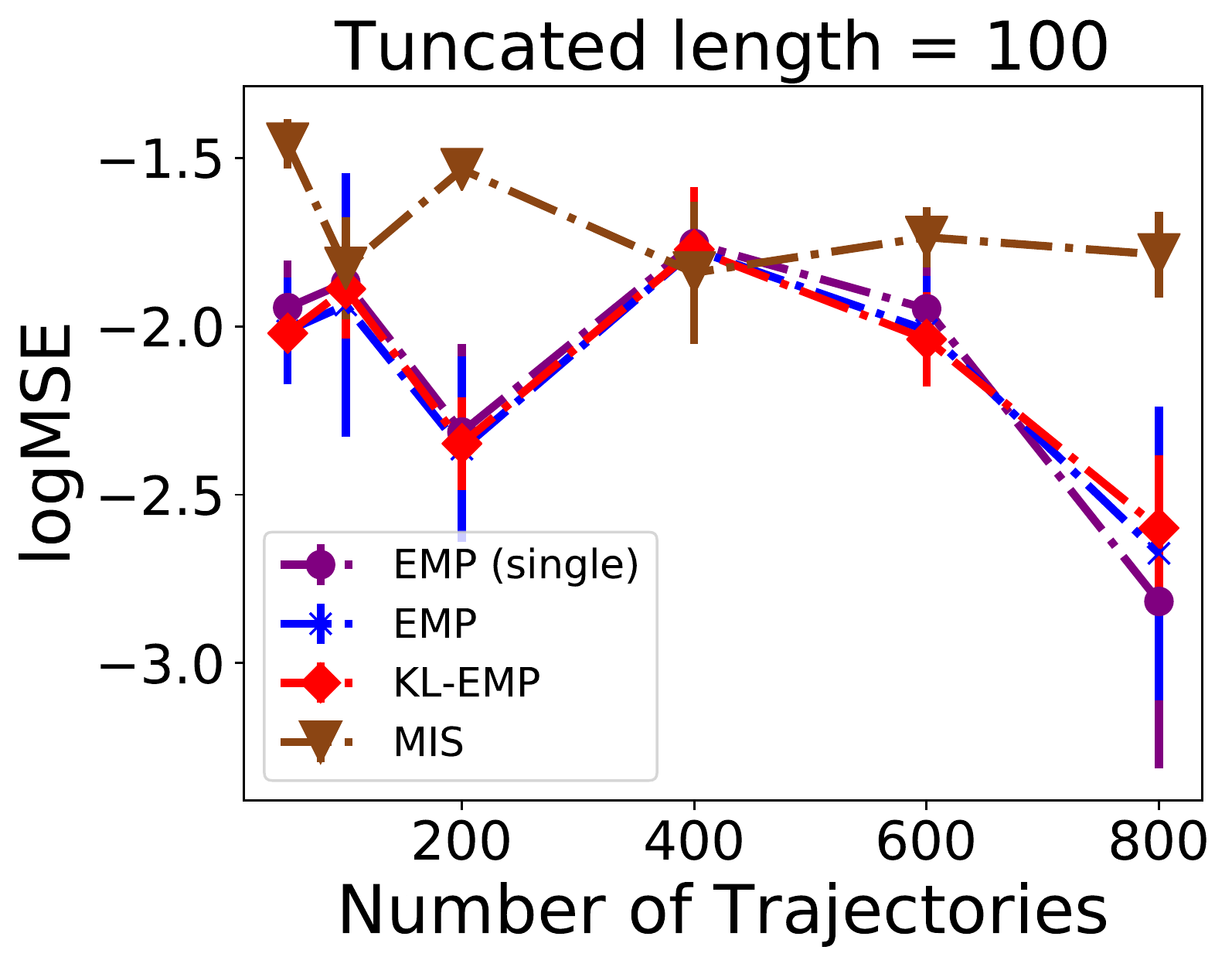}
		\label{fig:subf1}
		\includegraphics[ width=1.28in]{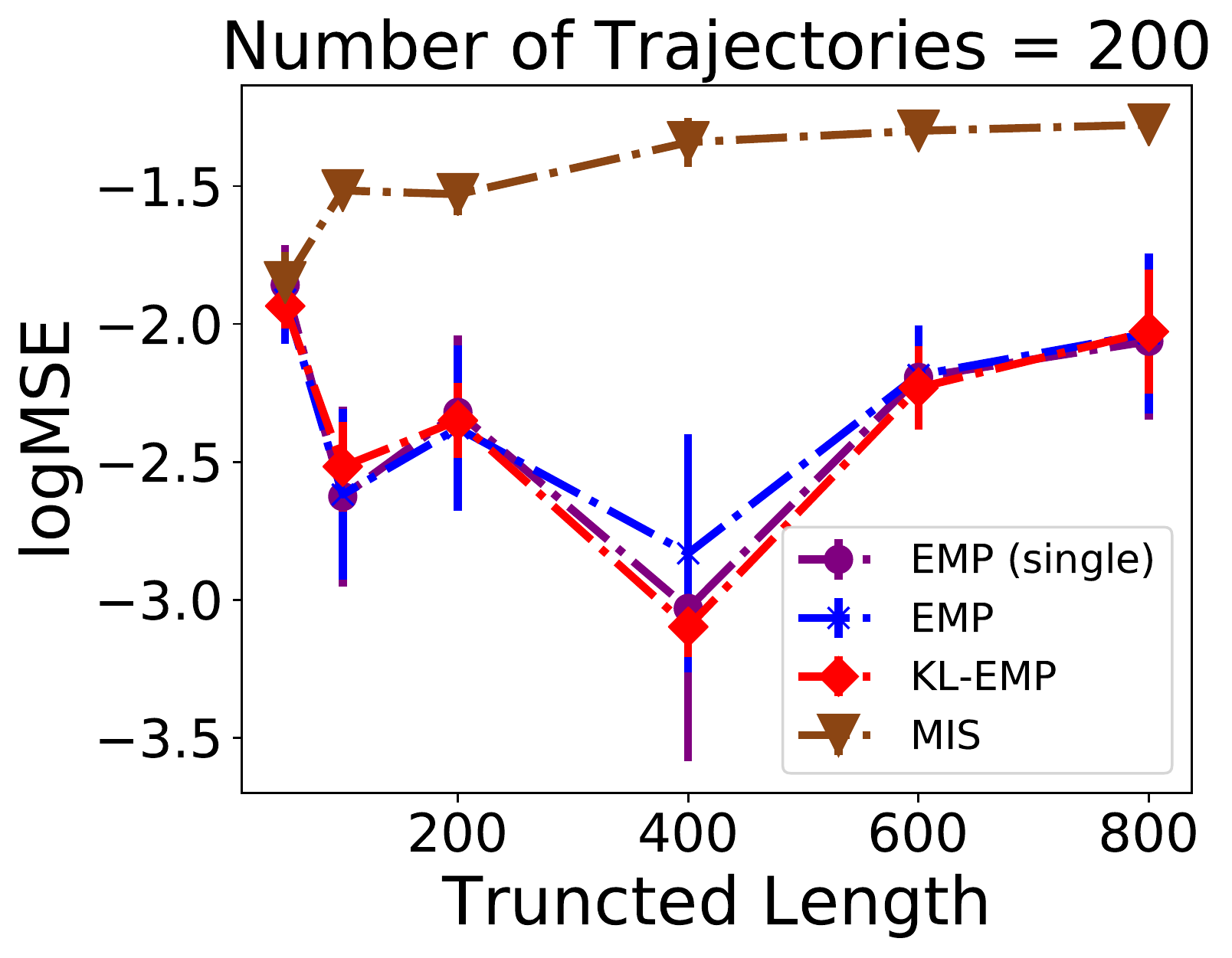}
		\label{fig:subfigure3}}

	\vspace{-.12in}
	\subfigure[Gridworld]{%
		\includegraphics[ width=1.2in]{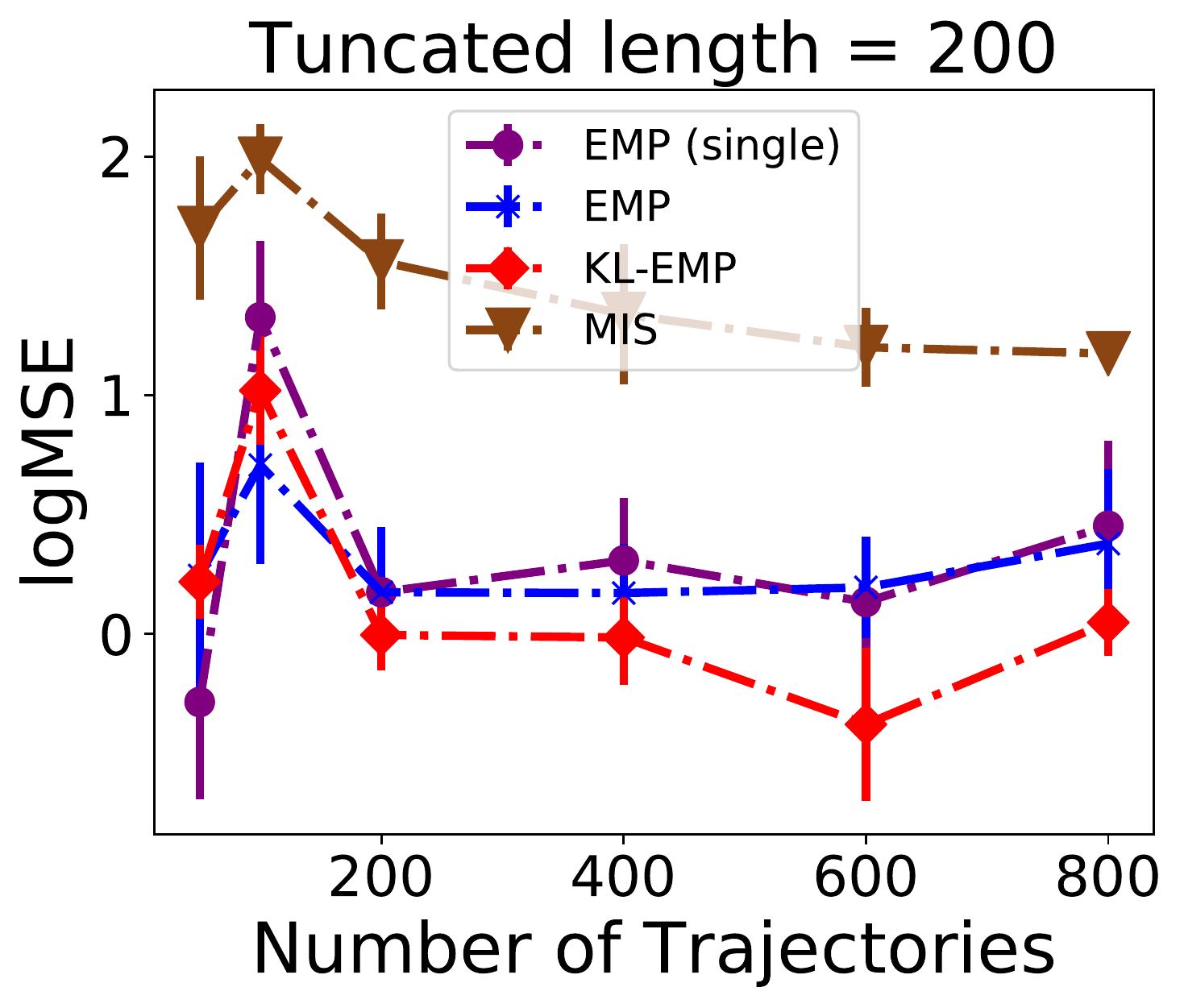}
		\label{fig:subfigure6}
		\includegraphics[ width=1.2in]{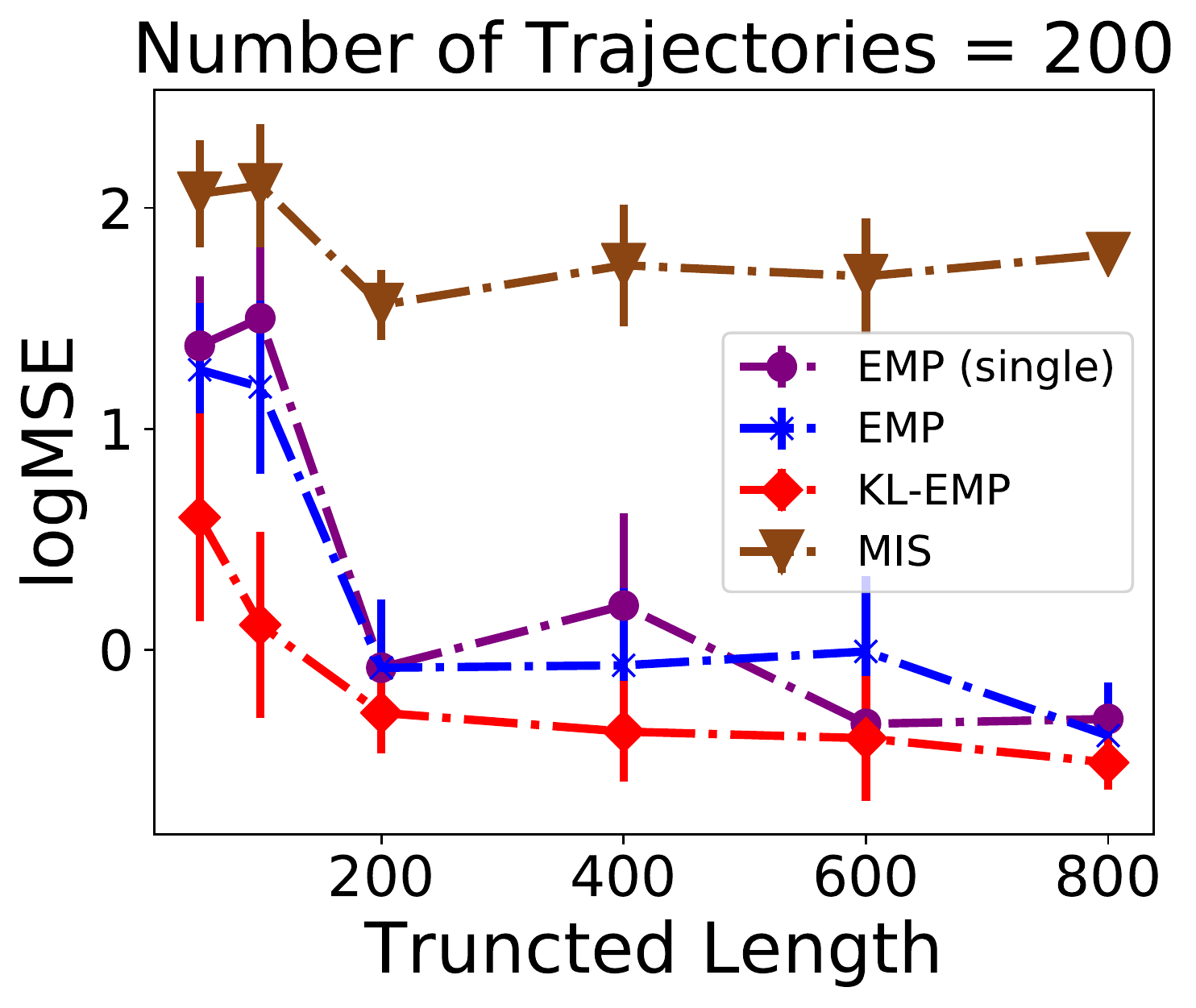}
		\label{fig:subfigure6}}
~
	\subfigure[Pendulum]{%
		\includegraphics[width=1.3in]{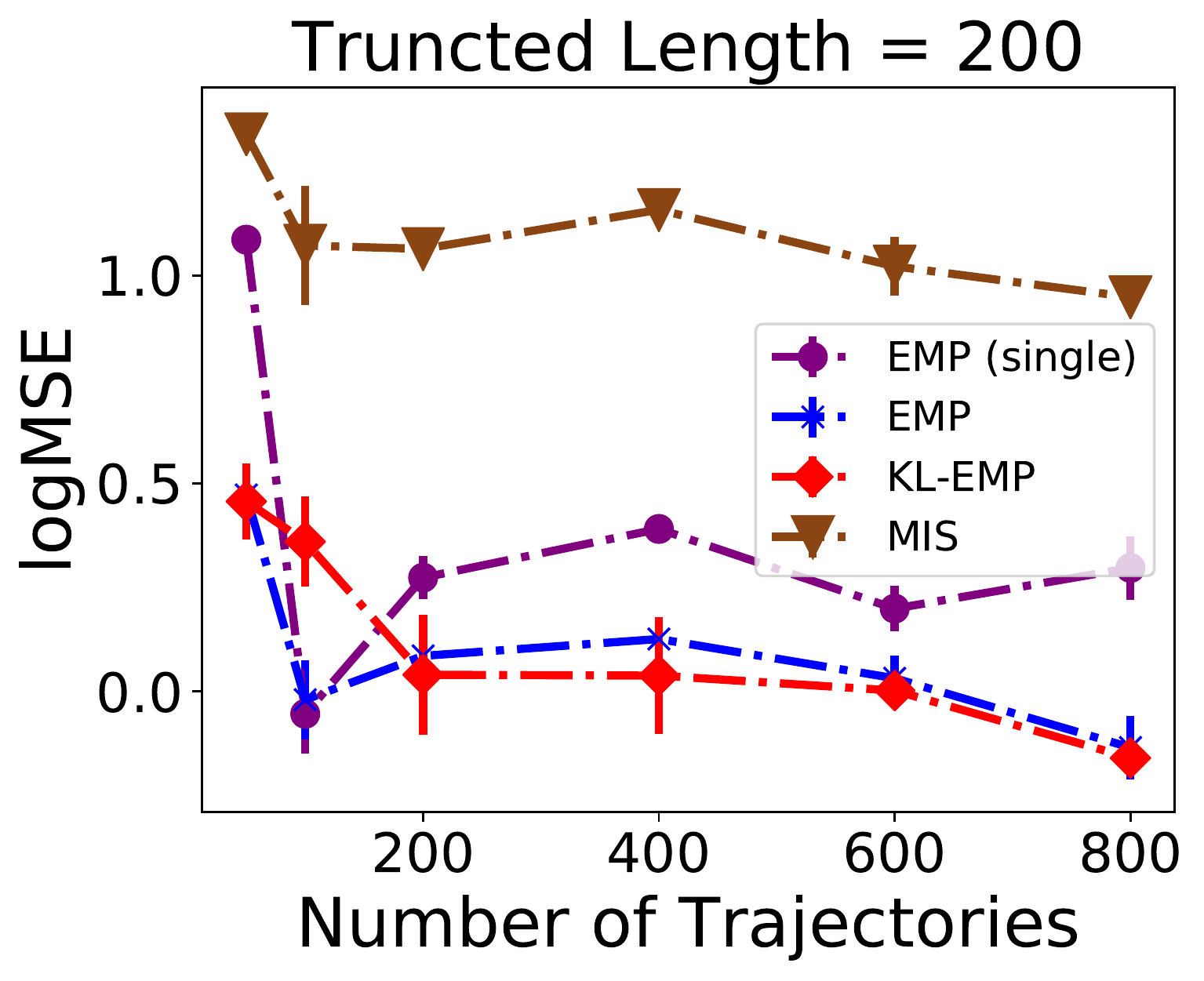}
		\label{fig:subfigure2}
		\includegraphics[width=1.3in]{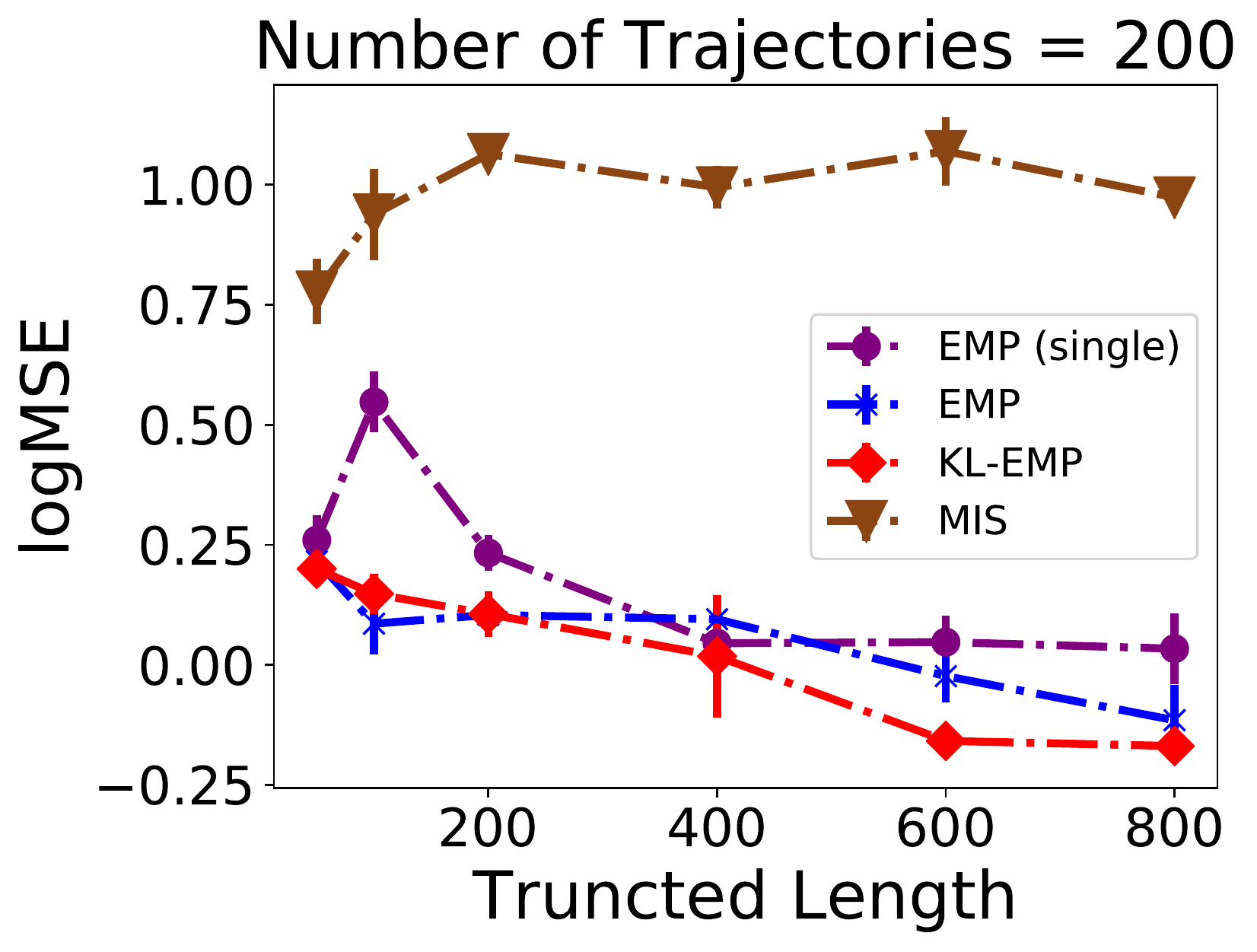}
		\label{fig:subfigure4}}
~
\vspace{-1em}
\caption{\label{fig:mis} Multiple-behavior-policy results of EMP (single), EMP, KL-EMP and MIS across continuous and discrete environments with average reward.}
\end{figure*}



\textbf{Policy Evaluation Performance.} 
 Figure~\ref{single-com} reports the MSE of policy evaluation by EMP, BCH and IS methods for the 4 different environments. We observe that, (i) EMP consistently obtains smaller MSE than the other two methods for different sample scales and different environments. (ii) The performance of EMP and BCH improves as the number of trajectories and length of horizons increase, while the IS method suffers from growing variance.  
our method correctly estimates the true density ratio over the state space. 
 
\textbf{Partially Policy-agnostic versus Policy-agnostic OPPE.}
Figure~\ref{scenario3} reports the comparison results for the \textit{policy-aware} BCH,  \textit{partially policy-agnostic} EMP and a \textit{policy-agnostic} method, which we call it state-action distribution learning (SADL) and whose formal formulation is given in Appendix \ref{appdx: sadl}. The results show that all three methods obtain improvement as the number of length of trajectories increase. Roughly speaking, both EMP and SADL outperform BCH. The policy-agnostic SADL is better than EMP in the cases of small sample size. But when the sample size increases so that the estimated behavior policy is more accurate, EMP gradually exceeds SADL.

\textbf{Remark: } In our implementation of SADL, we use the same min-max formulation and optimization solver as EMP so that the comparison could shed more lights on the impact of behavior policy information on the performance of off-policy policy evaluation. 
We will report the comparison result between EMP and DualDice once the code is released.

\subsection{Results for Multiple Unknown Behavior Policies} \label{sec: experiment multiple}

As for multiple behavior policies, we conduct experiments in policy-aware and partially policy-agnostic settings. We report the results of partially policy-agnostic setting in this section and the policy-aware setting is described in  Appendix~\ref{appdix:known_unknown}. Because partially policy-agnostic version consistently achieves better performance. 

\textbf{Experiment Set-up.} We implement the following 4 methods: (1) the proposed EMP; (2) the multiple importance sampling (MIS) method as in \cite{Tirinzoni2019} using balanced heuristics; (3) EMP (single), in which we apply EMP for each subgroup of samples generated by a single behavior policy to obtain one OPPE value and finally output their average; (4) KL-EMP, which is an ad-hoc improvement of EMP using more information on the behavior policies and whose implementation details are given in Appendix \ref{append:kl}. 

\textbf{Policy Evaluation Performance.} Figure~\ref{fig:mis} reports the log MSE of the 4 methods in different environments with different sample scales. It shows that the proposed EMP outperforms both MIS and EMP (single). It is interesting to note that in EMP (single), actually more information on the behavior policies is learned than in EMP, but the learned stationary distribution corrections are mixed with naively equal weights. So, the advantage of EMP over EMP (single) can be probably attributed to (1) the robustness due to less required information on behavior policies; (2) a near-optimal weighted average that is automatically learned by pooling together the samples from different behavior policies. 

On the other hand, we see that the performance of KL-EMP has greater improvement with the increase of sample size and eventually outperform EMP in cases of large sample size. This is because, KL-EMP replaces the fixed sample proportion (i.e. $w_j$ as defined in Section \ref{sec: mis analysis}) with a KL divergence-based proportion, which is better estimated with more data sample.

\section{Conclusion}

In this paper, we advocate the viewpoint of partial policy-awareness and the benefits of estimating a mixture policy
for off-policy policy evaluation. The theoretical results of reduced variance coupled with experimental results illustrate
the power of this class of methods. One key question that still remains is the following: if we are willing to estimate the individual behavior policies, can we further improve EMP by developing an efficient algorithm to compute the optimal weights? One other question is a direct comparison of DualDice and EMP when the code of DualDice is released, this will allow us to see the props and cons of
inductive bias offered by the Bellman equation used by DualDice and direct estimation of the mixture policy used by EMP.

\newpage
\bibliography{iclr2020_conference}
\bibliographystyle{iclr2020_conference}

\appendix
\newpage
\section{Kernel Method}\label{appdx: kernel}


We use the reproducing kernel Hilbert space to solve the mini-max problem of BCH~(\cite{liumethod}). The key property of RKHS we leveraged is called \textbf{reproducing property}. The reproducing property claims, for any function $f \in \mathcal{H}$ ($\mathcal{H}$ is a RKHS), the evaluation of $f$ at point x equals its inner product with another function in RKHS: $f(s) =\langle f,k(s,\cdot)\rangle_\mathcal{H}$.

Given the objective function of BCH $L(w,f) =\mathbb{E}_{(s,a,s')\sim d_{\pi_0}}[(\omega(s)\frac{\pi(a|s}{\pi_0(a|s)-\omega})f(s')]$. We use the reproducing property to obtain the closed form representation of $\max _{f \in \mathcal{F}} L(w, f)^{2}$, which is shown as follows:

$$\max _{f \in \mathcal{F}} L(w, f)^{2}=\mathbb{E}_{(s,a,s')\sim d_{\pi_0},(\bar s,\bar a,\bar s')\sim d_{\pi_0}}\left[\Delta\left(\omega ; s, a, s^{\prime}\right) \Delta\left(w ; \overline{s}, \overline{a}, \overline{s}^{\prime}\right) k\left(s^{\prime}, \overline{s}^{\prime}\right)\right]$$\\.


This equation has been proved in BCH~\cite{liumethod}.

\section{Proof of Theorem \ref{thm: BCD estimated}}\label{appdx: thm1}
\subsection{Assumptions}
In this appendix, we provide the mathematical details and proof of Theorem \ref{thm: BCD estimated}.
We first introduce some notations and assumptions. 

We assume the behavior policy $\pi_0(a|s)$ belongs to a class of policies $\Pi = \{\pi(\theta; a,s):\theta \in \mathcal{E}_\theta\}$, where $\mathcal{F}_\theta$ is the parameter space, i.e. there exists $\theta_0\in E_1$ such that $\pi_0(a|s) =\pi(\theta_0;a,s)$. The estimated behavior policy $\hat{\pi}_0=\pi_{\hat{\theta}}$ is obtained via maximum likelihood method, i.e. 
$$\hat{\theta}=\arg\max \sum_{n=0}^{N-1}\log(\pi(\theta; s_{n}, a_n)).$$
We assume central limit theorem holds for $\hat{\theta}$.
Recall that we have assumed in Section \ref{sec: BCD estimated} that the true stationary distribution correction $\omega(s)=\omega(\eta_0;s)$. Using the kernel method introduced in Appendix \ref{appdx: kernel}, our estimation $\hat{\omega}(s)=\omega(\hat{\eta};s)$ is obtained via
$$\min_\eta \sum_{0\leq i, j\leq N-1}G(\eta,\hat{\theta};x_i,x_j),$$
with $x_i = (s_i,a_i,s'_i)$ and
$$G(\eta,\hat{\theta};(x_i,x_j))=\left(\omega(\eta;s_i)\frac{\pi(a_i|s_i)}{\pi(\hat{\theta},a_i,s_i)}-\omega(\eta,s'_i)\right)\left(\omega(\eta;s_j)\frac{\pi(a_j|s_j)}{\pi(\hat{\theta},a_j,s_j)}-\omega(\eta,s'_j)\right)k(s'_i,s'_j).$$

\begin{assumption} \label{assmp: G} We assume the following regularity conditions on $G$:
	\begin{enumerate}
		\item $G$ is second order differentiable.
		\item $\mathbb{E}[\partial_\eta\partial_\theta G(\eta_0,\theta_0;x_i,x_j)]$ is finite. 
		\item  $\mathbb{E}[\partial^2_\eta G(\eta_0,\theta_0;x_i,x_j)]$ is finite and non-zero.
		\item $\mathbb{E}[\partial_\eta G(\eta_0,\theta_0;x_i,x_j)^2]$ is finite. 
	\end{enumerate}
\end{assumption}
Here we simply write $\mathbb{E}_{x_i\sim d_{\pi_0}, x_j\sim d_{\pi_0} }$ as $\mathbb{E}$ for the simplicity of notation.
\subsection{Proof of Theorem \ref{thm: BCD estimated}}
\begin{proof}
	Following the kernel method,
	\begin{align*}\hat{\eta}&=\arg\min_{\eta} =\arg\min_\eta
	\sum_{0\leq i,j\leq N-1}G(\eta,\hat{\theta};(x_i,x_j))\text{ with }x_i = (s_i,a_i,s'_i)\text{ and }s'_i \triangleq s_{i+1}.
	\end{align*}
	Then, $\sum_{1\leq i, j\leq N}\partial_\eta G(\hat{\eta},\hat{\theta};(x_i,x_j))=0$, we have
	\begin{align*}
	0 =& \frac{1}{N\sqrt{N}}\sum_{0\leq i, j\leq N-1}\partial_\eta G(\eta_0,\theta_0;(x_i,x_j))+\sqrt{N}(\hat{\eta}-\eta_0)\frac{1}{N^2}\sum_{0\leq i, j\leq N-1}\partial^2_\eta G(\eta_0,\theta_0;(x_i,x_j)) \\
	& +
	\sqrt{N}(\hat{\theta}-\theta)\frac{1}{N^2}\sum_{0\leq i, j\leq N-1}\partial_\theta\partial_\eta G(\eta_0,\theta_0;(x_i,x_j))\\
	=&\frac{1}{N\sqrt{N}}\sum_{0\leq i, j\leq N-1}\partial_\eta G(\eta_0,\theta_0;(x_i,x_j)) +\sqrt{N}(\hat{\eta}-\eta_0)\mathbb{E}\left[\partial^2_\eta G(\eta_0,\theta_0;(x_1,x_2))\right]\\
	&+ 	\sqrt{N}(\hat{\theta}-\theta)\mathbb{E}\left[\partial_\theta\partial_\eta G(\eta_0,\theta_0;(x_1,x_2))\right]+o_p(1) .
	\end{align*}
 Similarly, we have
	$$\tilde{\eta}
	=\arg\max_\eta
	\sum_{1\leq i,j\leq N}G(\eta,\theta_0;(x_i,x_j)),
	$$
	and
	$0 = \frac{1}{N\sqrt{N}}\sum_{0\leq i, j\leq N-1}\partial_\eta G(\eta_0,\theta_0;(x_i,x_j))+\sqrt{N}(\tilde{\eta}-\eta_0)\mathbb{E}\left[\partial^2_\eta G(\eta_0,\theta_0;(x_1,x_2))\right] + o_p(1).$
	Define 
	$S(\theta;(x_i,x_j))=\log(\pi(\theta;s_i,a_i))+\log(\pi(\theta;s_j,a_j)).$
	According to our estimation method, 
	$$\hat{\theta}=\arg\max_\theta \sum_{0\leq i, j\leq N-1} S(\theta;(x_i,x_j)).$$
	Therefore, 
	$0 = \frac{1}{N\sqrt{N}}\sum_{0\leq i, j\leq N-1} \partial_\theta S(\theta_0;(x_i,x_j))+\sqrt{N}(\hat{\theta}-\theta_0)\mathbb{E}\left[ \partial^2_\theta S(\theta_0;(x_1,x_2))\right]+o_p(1). $
	Following the proof of Theorem 1 of (Henmin et al. 2007), it suffices to prove that
	\begin{equation}\label{eq: proof1}
	\mathbb{E}\left[\partial _\theta\partial_\eta G(\eta_0,\theta_0;(x_1,x_2))\right]= \mathbb{E}\left[-\partial_\eta G(\eta_0,\theta_0;(x_1,x_2))\partial_\theta S(\theta_0;(x_1,x_2))\right].
	\end{equation}
	
	One can check
	\begin{align*}
	&\mathbb{E}\left[\partial_\eta G(\eta_0,\theta_0;(x_1,x_2))\right]\\
	=&\mathbb{E}\left[k(s_1',s_2')\left[\left(\partial_\eta \omega(\eta_0;s_1)\frac{\pi(a_1|s_1)}{\pi(\theta_0;a_1,s_1)}-\partial_\eta\omega(\eta_0;s'_1)\right)\left(\omega(\eta_0;s_2)\frac{\pi(a_2|s_2)}{\pi(\theta_0;a_2,s_2)}-\omega(\eta_0;s'_2)\right)\right.\right.\\
	&+\left.\left. \left(\partial_\eta \omega(\eta;s_2)\frac{\pi(a_2|s_2)}{\pi(\theta_0;a_2,s_2)}-\partial_\eta\omega(\eta_0;s'_2)\right)\left(\omega(\eta_0;s_1)\frac{\pi(a_1|s_1)}{\pi(\theta_0;a_1,s_1)}-\omega(\eta_0;s'_1)\right)\right]\right]\\
	= &\mathbb{E}\left[\left(k(s_1',s_2')+k(s_2',s_1')\right)\left(\partial_\eta \omega(\eta_0;s_1)\frac{\pi(a_1|s_1)}{\pi(\theta_0;a_1,s_1)}-\partial_\eta\omega(\eta_0;s'_1)\right)
	\left(\omega(\eta_0;s_2)\frac{\pi(a_2|s_2)}{\pi(\theta_0;a_2,s_2)}-\omega(\eta_0;s'_2)\right)\right].
	\end{align*}
	The last equality holds because $(x_1,x_2)\sim d_{\pi_0}(s_1)\pi(x_1;\eta_0)\otimes  d_{\pi_0}(s_2)\pi(x_2;\eta_0)$. Besides, we have
	$$
	\partial_\theta S(\theta;(x_1,x_2)) = \frac{\partial_\theta\pi(\theta;a_1,s_1)}{\pi(\theta;a_1,s_1)^2}+\frac{\partial_\theta\pi(\theta;a_2,s_2)}{\pi(\theta;a_2,s_2)^2}.$$
	Then, we derive
	\begin{align*}
	&\mathbb{E}\left[\partial _\theta\partial_\eta G(\eta_0,\theta;(x_1,x_2))\right]\\
	=&\mathbb{E}\left[\left(k(s'_1,s'_2)+k(s'_2,s'_1)\right) \left[-\partial_\eta \omega(\eta_0;s_1)\frac{\pi'(a_1|s_1)}{\pi(\theta_0; a_1,s_1)^2}\left(\omega(\eta_0;s_2)\frac{\pi(a_2|s_2)}{\pi(\theta_0;a_2,s_2)}-\omega(\eta_0;s'_2)\right)\right. \right.\\
	&\left.\left. -\omega(\eta_0;s_2)\frac{\pi'(a_2|s_2)}{\pi(\theta_0,a_2,s_2)^2}\left(\partial_\eta \omega(\eta_0;s_1)\frac{\pi(a_1|s_1)}{\pi(\theta_0;a_1,s_1)}-\partial_\eta\omega(\eta_0;s'_1)\right)\right]\right]\\
	=& \mathbb{E}\left[\left(k(s'_1,s'_2)+k(s'_2,s'_1)\right) \left[-\left(\partial_\eta \omega(\eta_0;s_1)\frac{\pi'(a_1|s_1)}{\pi(\theta_0;a_1,s_1)^2}-\frac{\partial_\eta \omega(\eta_0;s'_1)}{\pi(\theta_0;a_1,s_1)}\right)\left(\omega(\eta_0;s_2)\frac{\pi(a_2|s_2)}{\pi(\theta_0;a_2,s_2)}-\omega(\eta_0;s'_2)\right)\right.\right. \\
	&\left.\left. -\left(\omega(\eta_0;s_2)\frac{\pi'(a_2|s_2)}{\pi(\theta_0;a_2,s_2)^2}-\frac{\omega(\eta_0;s_j')}{\pi(\theta_0;a_2,s_2)}\right)\left(\partial_\eta \omega(\eta_0;s_1)\frac{\pi(a_1|s_1)}{\pi(\theta_0;a_1,s_1)}-\partial_\eta\omega(\eta_0;s'_1)\right)\right]\right]\\
	& -\mathbb{E}\left[\left(k(s'_1,s'_2)+k(s'_2,s'_1)\right)
	\left[\frac{\partial_\eta \omega(\eta_0;s'_1)}{\pi(\theta_0;a_1,s_1)}\left(\omega(\eta_0;s_2)\frac{\pi(a_1|s_1)}{\pi(\theta_0;a_2,s_2)}-\omega(\eta_0;s'_2)\right)\right.\right.\\ 
	&\left.\left.+ \frac{\omega(\eta_0;s'_2)}{\pi(\theta_0;a_2,s_2)}\left(\partial_\eta \omega(\eta_0;s_1)\frac{\pi(a_1|s_1)}{\pi(\theta_0;a_1,s_1)}-\partial_\eta\omega(\eta_0;s'_1)\right) \right]\right]\\
	\triangleq & \mathbb{E}\left[-\partial_\eta G(\eta_0,\theta_0;(x_1,x_2))\partial_\theta S(\theta_0;(x_1,x_2))\right] + \mathbb{E}\left[H(\eta_0,\theta_0;(x_1,x_2))\right].
	\end{align*}
	Here, we define
	\begin{align*}
	H(\eta_0,\theta_0,(x_1,x_2))= ~&\frac{\partial_\eta \omega(\eta_0;s'_1)}{\pi(\theta_0;a_1,s_1)}\left(\omega(\eta_0;s_2)\frac{\pi(a_2|s_2)}{\pi(\theta_0;a_2,s_2)}-\omega(\eta_0,s'_2)\right) \\ 
	&+ \frac{\omega(\eta_0;s'_2)}{\pi(\theta_0;a_2,s_2)}\left(\partial_\eta \omega(\eta_0;s_i)\frac{\pi(a_1|s_1)}{\pi(\theta_0;a_1,s_1)}-\partial_\eta\omega(\eta_0;s'_1)\right).
	\end{align*}
	Note that 
	\begin{align*}
	&\mathbb{E}\left[\left(\omega(\eta_0;s_2)\frac{\pi(a_2|s_2)}{\pi(\theta_0,a_2,s_2)}-\omega(\eta_0,s'_2)\right)|a_1,s_1,s_1',s'_2\right]=0,\\
	&\mathbb{E}\left[\left(\partial_\eta \omega(\eta_0;s_1)\frac{\pi(a_1|s_1)}{\pi(
		\theta_0;a_1,s_1)}-\partial_\eta\omega(\eta_0;s'_1)\right)| a_2,s_2,s_2'\right]= 0.
	\end{align*}
	Therefore $\mathbb{E}[H(\eta_0,\theta_0;(x_1,x_2))]=0$. So we obtain (\ref{eq: proof1}). 
\end{proof}
\subsection{Proof of Corollary \ref{crll: error bound}}
\begin{proof}
In the prove of Theorem \ref{thm: BCD estimated}, we see that 
$$	\frac{1}{N^2}\sum_{0\leq i, j\leq N-1}\partial_\eta G(\eta_0,\theta_0;(x_i,x_j)) +K_1(\hat{\eta}-\eta_0)
+ K_2(\hat{\theta}-\theta)=o_p(1).$$
with $K_1=\mathbb{E}\left[\partial^2_\eta G(\eta_0,\theta_0;(x_1,x_2))\right]$ and $K_2=\mathbb{E}\left[\partial_\theta\partial_\eta G(\eta_0,\theta_0;(x_1,x_2))\right]$. Therefore,
$$\mathbb{E}[(\hat{\eta}-\eta_0)^2] \leq 2K_1^{-2}\left(K_2^2\mathbb{E}[(\hat{\theta}-\theta_0)^2]+\mathbb{E}\left[\left(\frac{1}{N^2}\sum_{0\leq i, j\leq N-1}\partial_\eta G(\eta_0,\theta_0;(x_i,x_j))\right)^2\right]\right).$$
We assume that CLT holds for the maximum likelihood estimator $\hat{\theta}$, i.e. $\mathbb{E}[(\hat{\theta}-\theta_0)^2]=O(1/N)$. Besides, as $\mathbb{E}[\partial_\eta G(\eta_0,\theta_0;(x_i,x_j))]=0$, under Condition 4 of Assumption \ref{assmp: G}, , we can apply the central limit theorem (for stationary Markov chain) and have
$$\mathbb{E}\left[\left(\frac{1}{N\sqrt{N}}\sum_{0\leq i, j\leq N-1}\partial_\eta G(\eta_0,\theta_0;(x_i,x_j))\right)^2\right] = O(1).$$
Therefore,
$$\mathbb{E}[(\hat{\eta}-\eta_0)^2] = O(1/N).$$

\end{proof}


\section{Proofs of Propositions for EMP}\label{appdx: multiple}
\begin{proof}[Proof of Proposition \ref{prop: emp reward}]
	\begin{equation*}
	\begin{aligned}
	&\mathbb{E}_{(s,a)\sim d_0}\left[\omega(s)\frac{\pi(a|s)}{\pi_0(a|s)}r(s,a)\right]=\sum_{s,a}\omega(s)\frac{\pi(a|s)}{\pi_0(a|s)}r(s,a)\sum_j w_jd_{\pi_j}(s)\pi_j(a|s)\\
	=&\sum_{s,a}\omega(s)\frac{\pi(a|s)}{\pi_0(a|s)}r(s,a) d_0(s)\pi_0(s)= \sum_{s,a}d_0(s)\omega(s)\pi(a|s)r(s,a) = R_\pi.
	\end{aligned}
	\end{equation*}
\end{proof}
\begin{proof}[Proof of Proposition \ref{prop: emp omega}]
	If $\omega=d_\pi/d_0$, based on the stationary equation 
	$$d_\pi(s')=\sum_{s,a}d_\pi(s)\pi(a|s)T(s'|a,s),\text{ for any }s'\in S,$$ we have 
	\begin{align*}
	\sum_{s'}\omega(s')d_{0}(s')f(s')&=\sum_{s,a}\omega(s)d_0(s)\pi(a|s)T(s'|a,s)f(s')\\
	&= \sum_{j}\sum_{s,a} \omega(s) \frac{\pi(a|s)}{\sum_j w_jd_{\pi}(s)\pi_j(a|s)/d_0} w_jd_{\pi_j}(s)\pi_j(a|s)T(s'|a,s)f(s')\\
	& = \mathbb{E}_{(s,a,s')\sim d_0}\left[\omega(s)\frac{\pi(a|s)}{p_0(a|s)}f(s')\right].
	\end{align*}
	Therefore,
	$$ \mathbb{E}_{(s,a,s')\sim d_0}\left[\left(\omega(s')-\omega(s)\frac{\pi(a|s)}{p_0(a|s)}\right)f(s')\right]=0.$$
	On the opposite way, if 
	$$ \mathbb{E}_{(s,a,s')\sim d_0}\left[\left(\omega(s')-\omega(s)\frac{\pi(a|s)}{p_0(a|s)}\right)f(s')\right]=0, \text{ for all function }f: S\to \mathbb{R},$$
	we should have
	$$d_0(s')\omega(s')=\sum_{s,a}d_0(s)\omega(s)\pi(a|s)T(s'|a,s).$$
	Therefore, $d_0\omega$ satisfy the stationary equation and must equal to $d_\pi$ (up to a constant).
\end{proof}
\begin{proof}[Proof of Proposition \ref{thm: multiple estimated}]
	The proof follows immediately from that of Theorem \ref{thm: BCD estimated}. In particular, assume $\pi_0\in \{\pi(\theta;a,s):\theta\in\mathcal{E}_\theta\}$ and the estimated $\hat{\pi}_0=\pi(\hat{\theta};\cdot)$ is obtained via
	$$\hat{\theta}=\arg\max_\theta \sum_j\sum_{n=0}^{N_j-1} \log(\pi(\theta;s_{j,n},a_{j,n})).$$
	The rest part of the proof follows the same argument in the proof of Theorem \ref{thm: BCD estimated}.
\end{proof}
\section{State-Action Distribution Learning}\label{appdx: sadl}

 Here we propose a behavior-agnostic approach that evaluates the target policy through learning occupation distribution correction instead of stationary distribution correction. Recall that the occupation distribution the stationary distribution of the state-action pair $d_{\pi}(s)\pi(a|s)$. We define occupation distribution correction as $d_{\pi}(s)\pi(a|s)/(d_{\pi_0}(s)\pi_0(a|s))$. Since $\pi(a|s)$ is known, we denote $u(a,s)=d_{\pi}(s)/(d_{\pi_0}(s)\pi_0(a|s))$ and formulate a min-max problem to learn $u$.



Following this notation, Equation (\ref{eq:stationary}) can be written as:
$$
u(s',a')d_{\pi_0}(s',a')= \sum_{s,a}u(s,a)\pi(a|s)d_{\pi_0}(s,a)P(s'|a,s), \quad \forall s',\forall a'
$$
For any test function $f(s',a'): S\times R \to \mathbb{R}$ and any probability density function $\omega(a')$ with respect to $ a' \in A$, the equation above implies that:
$$
\sum_{s',a'}u(s',a')\omega(a')f(s',a')d_{\pi_0}(s',a') = \sum_{s'}\sum_{a'}f(s',a')\omega(a')\sum_{s,a}u(s,a)\pi(a|s)d_{\pi_0}(s,a)P(s'|a,s),  \\
$$
which is equivalent to,


\[
\mathbb{E}_{(s,a,s')\sim d_{\pi_0}}[u(s,a)\omega(a)f(s,a)-u(s,a)\pi(a|s)\mathbb{E}_{a'\sim \omega}[f(s',a')]\big] =0, \quad \forall f.
\]

This suggest the following mini-max problem can be used to estimate $u(s,a)$
\begin{gather*}
\min \limits_{u}\{D(u):=\max \limits_{f\in\mathcal{F}}L(u,f)^2\},\\
\mathrm{where}\ \  L(u,f):=\mathbb{E}_{(s,a,s')\sim d_{\pi_0}}\big[u(s,a)\omega(a)f(s,a)-u(s,a)\pi(a|s)\mathbb{E}_{a'\sim \omega}[f(s',a')]\big].
\end{gather*}

We simplify this mini-max problem into a minimization problem using the kernel method. Theorem \ref{scen2_RKHS} gives the closed form representation of $D(u)$ when $\mathcal{F}$ is a unit ball in a RKHS with kernel $k$.

\begin{thm}\label{scen2_RKHS}
Assume $\mathcal{H}$ is a RKHS of functions $f(s,a)$ with a positive definite kernel $k((s,a),(s',a'))$, and define $\mathcal{F}:=\{f\in \mathcal{H}: ||f||_{\mathcal{H}}\leq 1\}$ to be the unit ball of $\mathcal{H}$. We have $\max \limits_{f\in\mathcal{F}}L(u,f)^2=$
$$
\begin{aligned}\label{scen2_RKHS}
        &=\mathbb{E}_{(s,a,s')\sim d_{\pi_0}, a' \sim \omega,(\bar s,\bar a,\bar s')\sim d_{\pi_0}, \bar a' \sim \omega}[u(s,a)u(\bar s, \bar a)\omega(a)\omega(\bar a)k((s,a),(\bar s, \bar a))\\
       &+u(s,a)u(\bar s, \bar a)\pi(a|s)\pi(\bar a|\bar s)k((s',a'),(\bar s', \bar a'))\\
       &-u(s,a)u(\bar s, \bar a)\omega(a)\pi(\bar a|\bar s)k((s,a),(\bar s', \bar a'))-u(s,a)u(\bar s, \bar a)\omega(\bar a)\pi( a| s)k((\bar s,\bar a), (s',a'))].  
\end{aligned}
$$
\end{thm}


\section{Experimental Details}
\subsection{Environment Description}\label{app:env}
\textbf{Taxi}  Taxi~\cite{dietterich2000hierarchical} is a 5 $\times$ 5 grid world simulating a taxi movement. Six actions are contained in Taxi: moves North, East, South, West, pick up and drop off a passenger. A reward of 20 is received when it picks up a passenger or drops her/he off at the right place, and a reward of -1 for each time step. The passengers are allow to randomly appear and disappear at every corner of the map at each time step. The 5 $\times$ 5 grid size yields 2000 states in total (25 $\times$ 24 $\times$ 5, corresponding to 25 taxi locations, 24 passenger appearance status and 5 taxi status (empty or with one of 4 destinations)).

\textbf{Gridworld}
 Gridworld~\cite{thomas2016data} is a 4 $\times$ 4 grid world which including one reward state, one terminate state and one fire state and thirteen normal state. Four action can be taken in this environment: up, down, left and right. A reward of -1 will be received while the agent in normal states, 1 reward is obtained in reward state, 100 reward is got in terminate state and -11 reward will got in fire state. 

\textbf{SinglePath}
This environment has 5 states, 2 actions. The agent begins in state 0 and both actions either take the agent from state n to state n + 1 or cause the agent to remain in state n. If the agent arrives at a new state, it will receive a +1 reward, otherwise it will get a 1 reward.

\textbf{Pendulum}
Pendulum has a continuous state space of $\mathcal{R}^3$ which describes the triangle of and a action space of $[-2,2]$.

\subsection{Behavior Policy Estimation}\label{appdix:esti}
EMP employs the maximum likelihood method as in \citet{Tirinzoni2019} to estimate the mixed policy $\hat{\pi}_0$ as
\begin{equation}
    \hat{\pi}_0 = \argmax_{\pi\in \Pi}\sum_{j=1}^m\sum_{n=1}^{N_j}\log \pi(a_n|s_n)
\end{equation}
As for discrete control tasks, the optimal $\hat{\pi}_0$ coincides with the count-frequency.

\subsection{Computation of KL-Weights}\label{append:kl}
In an ad-hoc way, we optimize the weights $w_j$ according to the KL-divergence between the behavior policies and the target policy, which can be estimated directly from the data. For finite-state space, we propose to choose 
\begin{equation}\label{eq: KL weight}
\begin{aligned}
    w^{KL}_j & = \frac{\sum_{s\in S} \mathbf{1}(j=\arg\min_k D_{KL}(\pi(\cdot|s)||\pi_{k}(\cdot|s)))}{\sum_{i=1}^m \sum_{s\in S} \mathbf{1}(i=\arg\min_k D_{KL}(\pi(\cdot|s)||\pi_{k}(\cdot|s)))} \\& = \frac{\sum_{s\in S} \mathbf{1}(j=\arg\min_k D_{KL}(\pi(\cdot|s)||\pi_{k}(\cdot|s)))}{|S|}
\end{aligned}
\end{equation}
To implement this method for infinite- or continuous-state space in the numerical experiments, we replace the set of all possible states $S$ in (\ref{eq: KL weight}) with the set of all states that has been visited in the data buffer.
The numerical results show that using the KL weights $\{w^{KL}_j\}$ could achieve smaller MSE compared to using $\{w_j\}$ as given by the data sample. We believe this approach deserves more careful analysis in future research studies.
\subsection{ Additional Experiment Results }\label{appdix:known_unknown}
BCH, EMP and KL-EMP have both policy-aware and partially policy-agnostic versions in multiple behavior policies. In policy-aware BCH we first apply BCH for each subgroup of samples generated by each behavior policy followed by output their average value. As for partially policy-agnostic BCH, it is equal to EMP (single).

The policy-aware version of EMP is named as BCH (pooled).
In BCH (pooled), the corresponding min-max problem formation is
\begin{equation}\label{eq: multiple known}
\min_\omega\max_f ~\mathbb{E}_{(j,s,a,s')\sim \mathcal{D}}\left[\left(\omega(s')-\omega(s)\frac{\pi(a|s)}{\pi_j(a|s)}\right)f(s')\right].
\end{equation}
They both pool the data from different behavior policies together and the main difference is that BCH (pooled) uses the exact behavior policies. 

The policy-aware version of KL-EMP is called BCH (KL-polled). The main difference between BCH (KL-polled) and BCH (polled) is that BCH (KL-polled) utilizes KL-divergence to calculate the weights.


The results of the two versions are shown in Figure~\ref{fig:known_unknown_mis}.
We observe that the partially policy-agnostic version OPPE consistently outperform the policy-aware version OPPE.
\begin{figure*}
	\centering
	\hspace{-1.em}
	\subfigure[Taxi]{%
		\includegraphics[ width=1.24in]{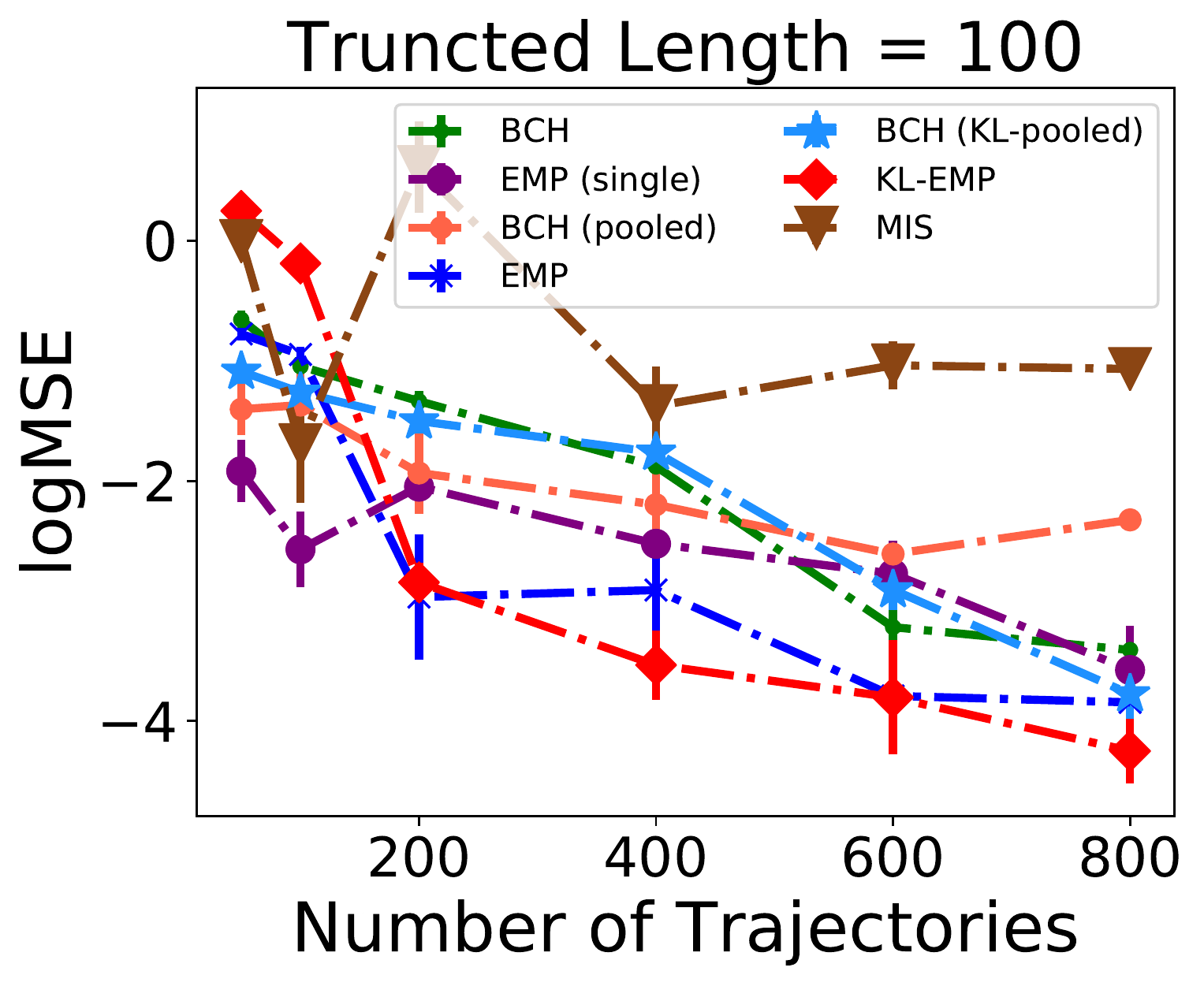}
		\label{fig:subfigure5}
		\includegraphics[ width=1.24in]{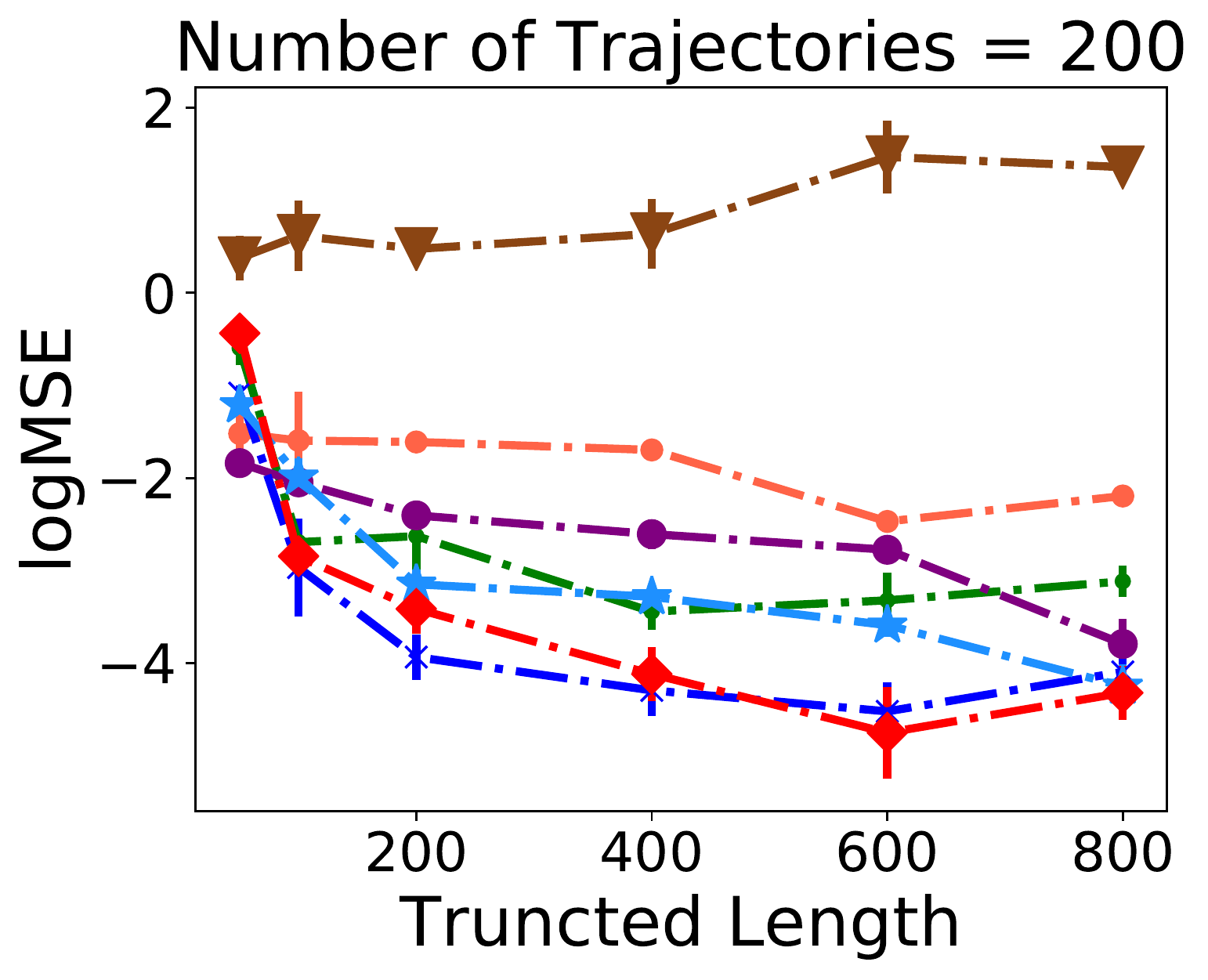}
		\label{fig:subfigure5}}
~
	\subfigure[Singlepath]{%
		\includegraphics[width=1.28in]{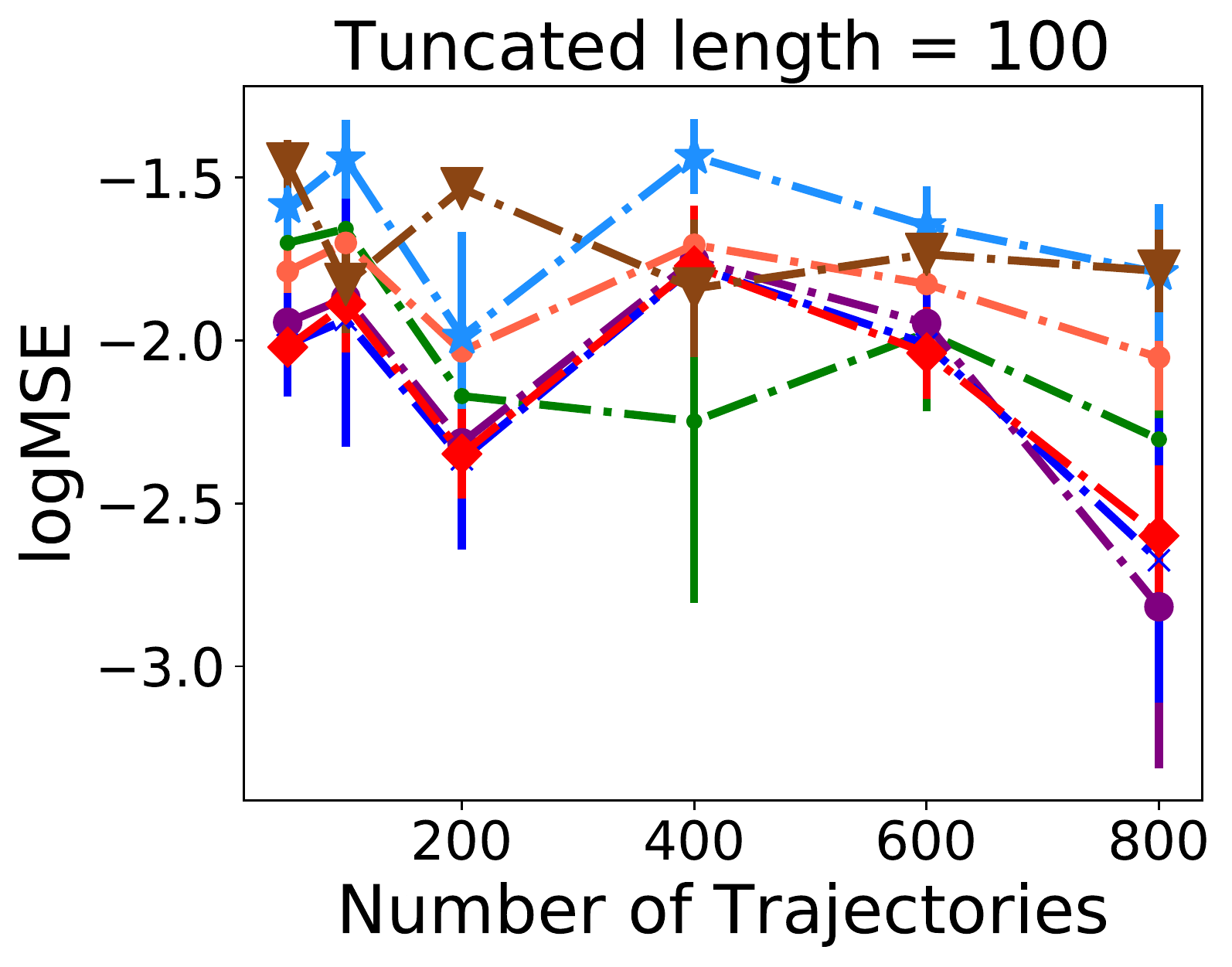}
		\label{fig:subf1}
		\includegraphics[ width=1.28in]{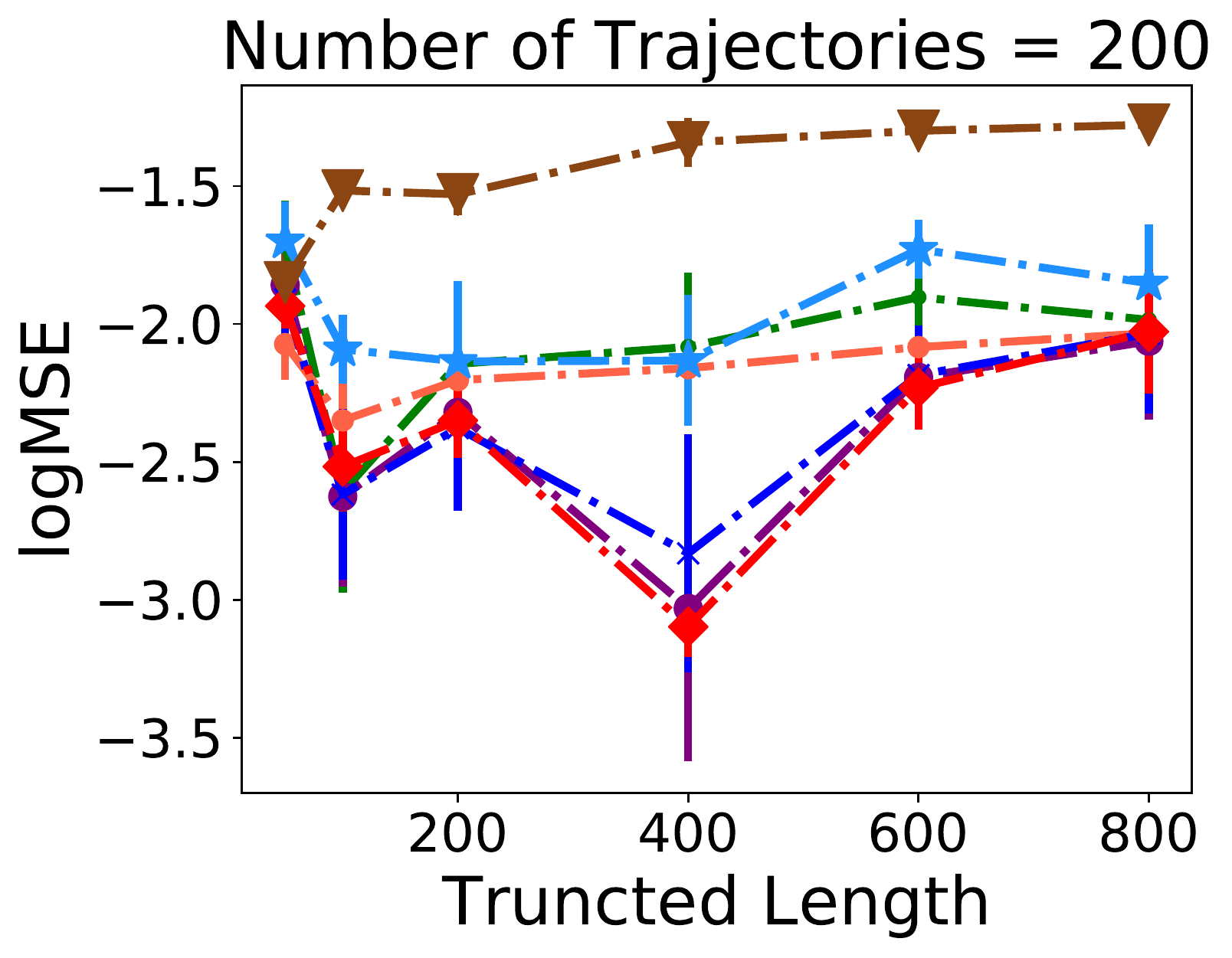}
		\label{fig:subfigure3}}

	\vspace{-.12in}
	\subfigure[Gridworld]{%
		\includegraphics[ width=1.2in]{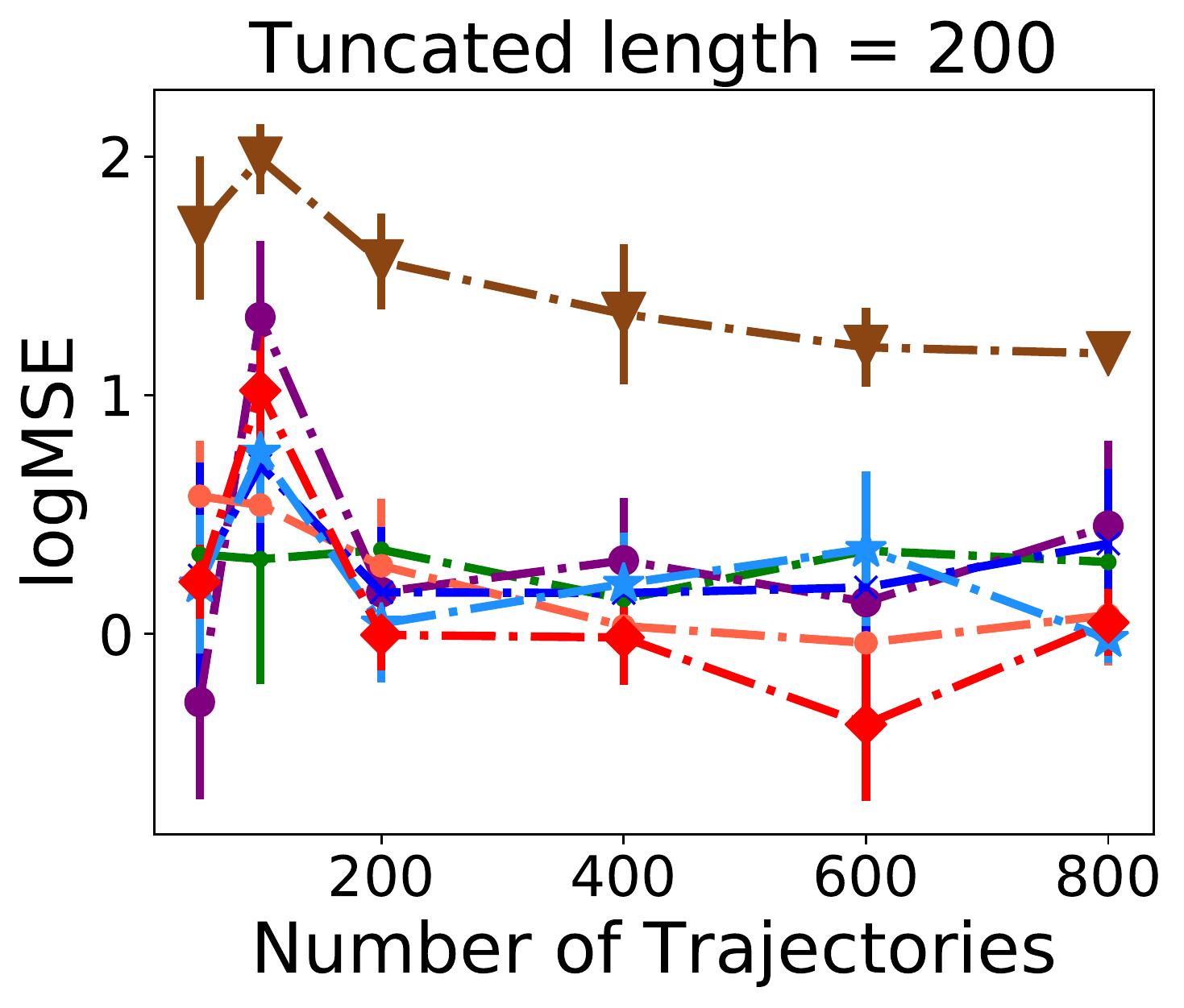}
		\label{fig:subfigure6}
		\includegraphics[ width=1.2in]{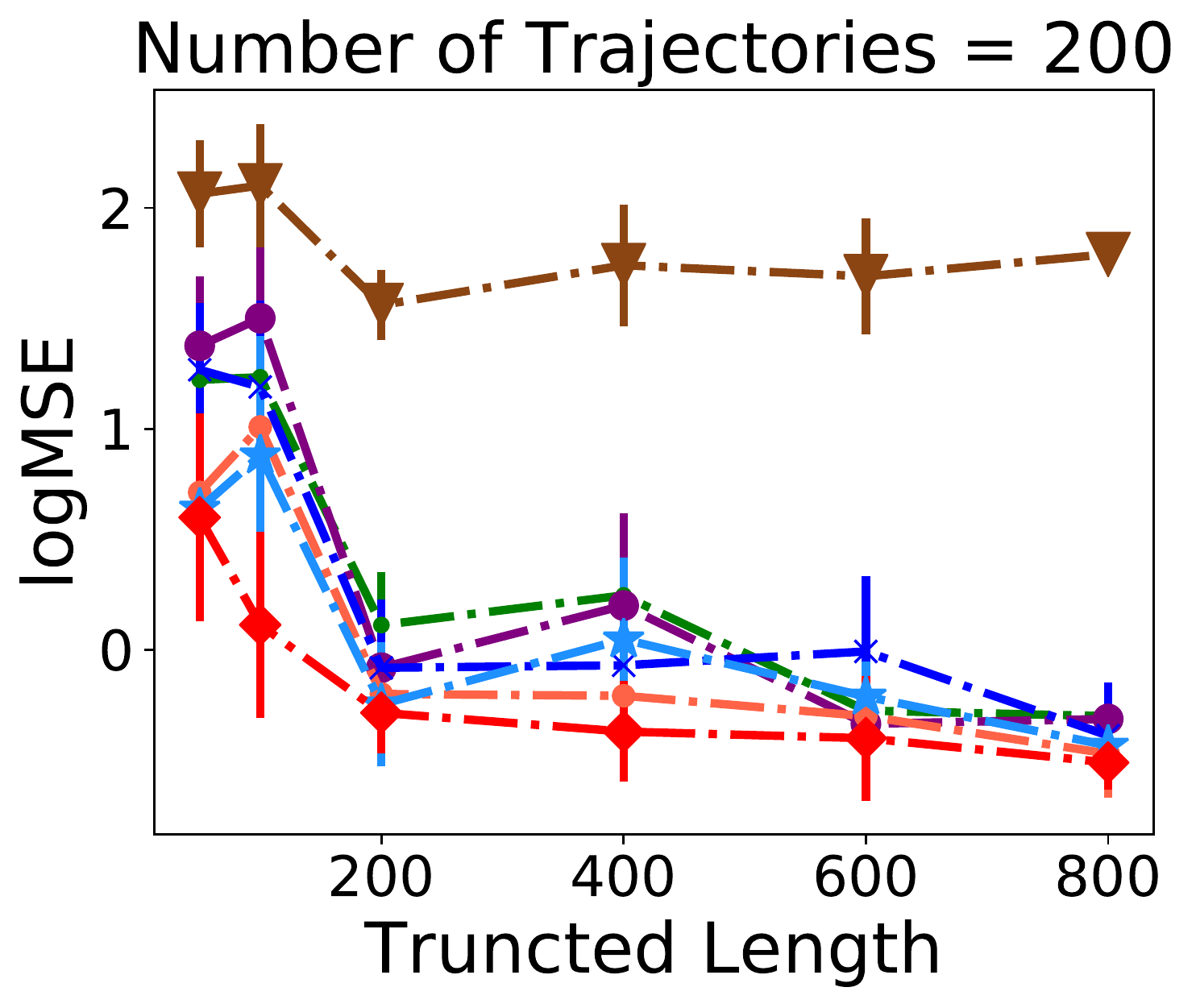}
		\label{fig:subfigure6}}
~
	\subfigure[Pendulum]{%
		\includegraphics[width=1.24in]{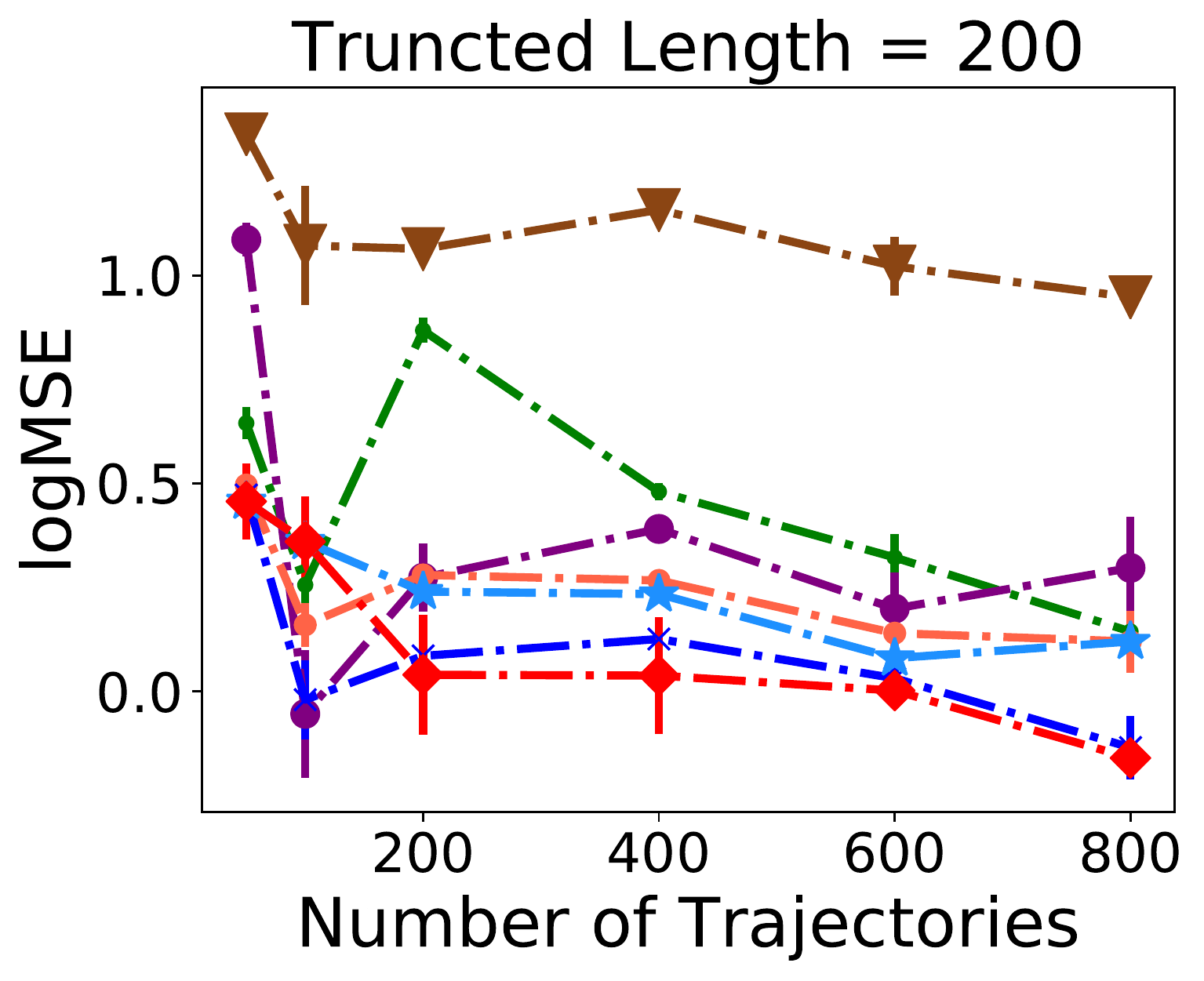}
		\label{fig:subfigure2}
		\includegraphics[width=1.3in]{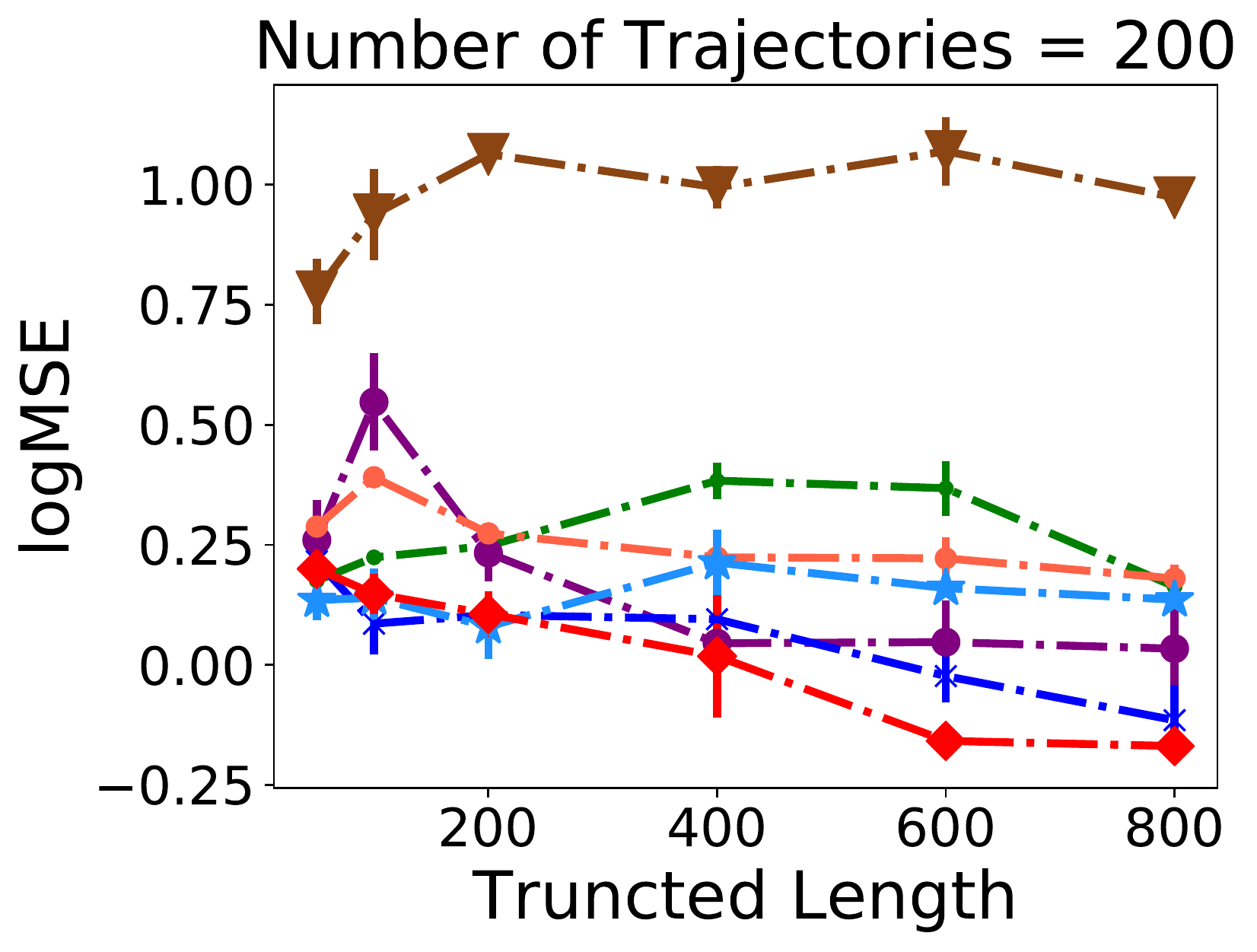}
		\label{fig:subfigure4}}
~

\caption{\label{fig:known_unknown_mis} Results of policy-aware OPPE methods (BCH, BCH (pooled) and BCH (KL-pooled)) and their corresponding partially policy-agnostic version (EMP (single), EMP and KL-EMP ) across continuous and discrete environments with average reward. }
\end{figure*}
\end{document}